\documentclass[nohyperref]{article}

\usepackage{microtype}
\usepackage{graphicx}
\usepackage{booktabs} 

\usepackage[accepted]{icml2022}
\usepackage{amsfonts}
\usepackage{amsmath}
\usepackage{amssymb}
\usepackage{amsthm}
\usepackage{mathtools}
\usepackage{enumerate}
\usepackage[font=small,labelfont=small]{caption}
\usepackage[font=footnotesize,labelfont=footnotesize]{subcaption}

\usepackage[utf8]{inputenc} 
\usepackage[T1]{fontenc}    
\usepackage{hyperref}       
\usepackage{url}            

\usepackage{nicefrac}       
\usepackage{bbm}
\usepackage{bm}
\usepackage{comment}

\usepackage{microtype}      
\usepackage{xcolor}         
\usepackage{amsmath,amssymb,framed, amsthm}
\usepackage{sidecap}

\DeclareMathOperator*{\argmax}{arg\,max}
\DeclareMathOperator*{\argmin}{arg\,min}
\DeclareMathOperator*{\bary}{Bary}
\DeclareMathOperator*{\E}{\mathbb{E}}
\DeclareMathOperator*{\V}{Var}

\DeclareMathOperator*{\id}{Id}
\DeclareMathOperator{\Tr}{Tr}
\DeclareMathOperator{\Pac}{\mathcal{P}_{2,ac}(\mathbb{R}^d)}
\DeclareMathOperator*{\hatlambda}{\hat{\lambda}}

\newtheorem{theorem}{Theorem}
\newtheorem{corollary}{Corollary}

\newtheorem{proposition}{Proposition}
\theoremstyle{definition}

\newtheorem{definition}{Definition}

\newif\ifcomments
\commentstrue

\ifcomments
    \usepackage{ulem}
    \def\mwedit#1{{$\!$\color{magenta} [MW: #1]}}
    
\else 
    \def\mwedit#1{}
    
\fi

\ifcomments
    \usepackage{ulem}
    \def\saedit#1{{$\!$\color{blue} [SA: #1]}}
    
\else 
    \def\saedit#1{}
    
\fi

\ifcomments
    \usepackage{ulem}
    \def\jmedit#1{{$\!$\color{brown} [JM: #1]}}
    
\else 
    \def\jmedit#1{}
    
\fi

\begin{document}

%

%

\twocolumn[

\icmltitle{Measure Estimation in the Barycentric Coding Model}

\begin{icmlauthorlist}

\icmlauthor{Matthew Werenski}{CS}
\icmlauthor{Ruijie Jiang}{ECE}
\icmlauthor{Abiy Tasissa}{math}
\icmlauthor{Shuchin Aeron}{ECE}
\icmlauthor{James M. Murphy}{math}
\end{icmlauthorlist}

\icmlaffiliation{CS}{Department of Computer Science, Tufts University } 
\icmlaffiliation{math}{Department of Mathematics, Tufts University } 
\icmlaffiliation{ECE}{Department of Electrical and Computer Engineering, Tufts University }

\icmlcorrespondingauthor{Matthew Werenski}{matthew.werenski@tufts.edu}

\icmlkeywords{...}

\vskip 0.3in
]

\printAffiliationsAndNotice

\begin{abstract}  This paper considers the problem of measure estimation under the \emph{barycentric coding model (BCM)}, in which an unknown measure is assumed to belong to the set of Wasserstein-2 barycenters of a finite set of known measures. Estimating a measure under this model is equivalent to estimating the unknown barycentric coordinates.  We provide novel geometrical, statistical, and computational insights for measure estimation under the BCM, consisting of three main results.  Our first main result leverages the Riemannian geometry of Wasserstein-2 space to provide a \emph{procedure for recovering the barycentric coordinates as the solution to a quadratic optimization problem} assuming access to the true reference measures.  The essential geometric insight is that the parameters of this quadratic problem are determined by inner products between the optimal displacement maps from the given measure to the reference measures defining the BCM.  Our second main result then establishes an \emph{algorithm for solving for the coordinates in the BCM when all the measures are observed empirically via i.i.d. samples}.  We prove precise \emph{rates of convergence} for this algorithm---determined by the smoothness of the underlying measures and their dimensionality---thereby guaranteeing its statistical consistency.  Finally, we demonstrate the utility of the BCM and associated estimation procedures in three application areas: (i) \emph{covariance estimation} for Gaussian measures; (ii) \emph{image processing}; and (iii) \emph{natural language processing}.
\end{abstract}

\section{Introduction}

A number of recent machine learning applications including computer vision \cite{schmitz2018wasserstein, bonneel2016coordinates}, domain adaptation \cite{montesuma2021wasserstein, redko2019optimal}, natural language processing (NLP) \cite{singh2020context, colombo2021automatic, xu2018distilled}, and unsupervised segmentation of multivariate time series data \cite{cheng2021dynamical} have shown the utility of representing and modeling high-dimensional data as probability distributions.  The essential insight in these applications is to utilize the Riemannian geometry of the space of probability distributions induced by the Wasserstein-2 metric and the Wasserstein-2 barycenters as alternatives to the Euclidean distance and linear combinations, respectively.  

In this context, the methods that use Wasserstein barycenters for data modeling and inference involve solving two core problems which we term the \emph{synthesis problem} and \emph{analysis problem}. To be precise, let $\Pac$ denote the space of absolutely continuous distributions on $\mathbb{R}^d$ with finite second moment.  For $\mu, \nu \in \Pac$,
\begin{equation} \label{eq:wass_dist}
    W_2^2(\mu,\nu) \triangleq \min_{T\#\mu = \nu} \int_{\mathbb{R}^{d}} ||T(x) - x||_2^2 d\mu(x)
\end{equation}
defines the \emph{Wasserstein-2 metric} where the minimum is over all measurable maps $T:\mathbb{R}^d \rightarrow \mathbb{R}^d$ and the \emph{pushforward} $T\#\mu = \nu$ is such that for all Borel sets $B$ we have $\nu[B] = \mu[T^{-1}(B)]$  \citep{villani2003topics, santambrogio2015optimal}. Under the assumption $\mu,\nu\in\Pac$, unique minimizing $T$ are known to exist and the resulting $W_{2}$ distance is finite \citep{villani2003topics}.  We refer to the space $\Pac$ equipped with the metric $W_{2}$ as the \emph{Wasserstein-2 space}, denoted $\mathcal{W}_{2}$.

Let $\Delta^p = \left \{\lambda=(\lambda_{1},\dots,\lambda_{p}) \in \mathbb{R}^p : \lambda_{i} \geq 0, \sum_{i=1}^p \lambda_i = 1 \right \}$ and let $\{\mu_{i}\}_{i=1}^{p}$ be known \emph{reference measures}. The \emph{$W_{2}$ barycenter} for coordinates $\lambda\in\Delta^{p}$ with respect to $\{\mu_{i}\}_{i=1}^{p}$ is defined as
\begin{equation} \label{eq:barycenter}
    \nu_{\lambda} \triangleq \argmin_{\nu\in \Pac} \frac{1}{2}\sum_{i=1}^p \lambda_i W_2^2(\nu, \mu_i).
\end{equation}
The measure $\nu_{\lambda}$ is an advection of $\{\mu_i\}_{i=1}^{p}$ and blends their geometric features; it may be thought of as a ``displacement interpolation" \citep{mccann1997convexity, villani2003topics}.  In contrast, the linear mixture $\sum_{i=1}^{p} \lambda_i \mu_i$ simply sets $\lambda_{i}$ to be the rate we draw samples from $\mu_i$.  Figure \ref{fig:bc_vs_lin} demonstrates the important difference between linear mixture models and Wasserstein barycenters on Gaussian distributions: the linear mixture is a standard Gaussian mixture model \citep{murphy2012machine}, while the Wasserstein barycenter is itself Gaussian. The important takeaway is that Wasserstein barycenters offer a more geometrically meaningful interpolation between measures compared to mixtures, and as indicated above have been successfully leveraged in a number of applications.

Now, let $\bary(\{\mu_i\}_{i=1}^p) \triangleq \left \{ \nu_{\lambda}: \lambda \in \Delta^p \right \}$ denote the set of all of possible barycenters of $\{\mu_i\}_{i=1}^p$. The \emph{synthesis} problem (\ref{eq:barycenter}) seeks to combine the reference measures $\{\mu_{i}\}_{i=1}^{p}$ in a notion of weighted average, with relative contributions determined by $\lambda$.  This problem is well-studied \citep{mccann1997convexity, agueh2011barycenters, kroshnin2019complexity} and known to have a unique solution under mild assumptions on the $\{\mu_{i}\}_{i=1}^{p}$.  Moreover, there exist fast algorithms for computing or approximating $\nu_{\lambda}$ when samples from the $\{\mu_{i}\}_{i=1}^{p}$ are available \citep{rabin2011wasserstein,cuturi2014fast,bonneel2015sliced, claici2018stochastic, yang2021fast}, which have spurred the application of Wasserstein barycenters to machine learning problems.

In this paper, our focus is on the less well-studied \emph{analysis} problem: given $\mu_0$ and reference measures $\{\mu_{i}\}_{i=1}^{p}$, solve
\begin{equation} \label{eq:analysis}
    \argmin_{\lambda\,\in\Delta^{p}} W_2^2\left ( \mu_0, \nu_{\lambda} \right ).
\end{equation}
Instead of \emph{combining} the measures $\{\mu_{i}\}_{i=1}^{p}$ according to the weights $\lambda\in\Delta^{p}$ as in (\ref{eq:barycenter}), the analysis problem (\ref{eq:analysis}) \emph{decomposes} a given measure $\mu_{0}$ into its optimal representation in the set $\bary(\{\mu_{i}\}_{i=1}^{p})$ as parameterized by the learned barycentric coordinates $\lambda\in\Delta^{p}$.  If the measure $\mu_0\in\bary(\{\mu_i\}_{i=1}^p)$, then $\min_{\lambda\,\in\Delta^{p}} W_2^2 ( \mu_0, \nu_{\lambda})=0$ and solving (\ref{eq:analysis}) is equivalent to learning the barycentric coordinates $\lambda\in\Delta^{p}$ such that $\nu_{\lambda}=\mu_{0}$. We refer to this general model of parameterizing measures $\mu_0 \in \bary(\{\mu_i\}_{i=1}^p)$ through the associated coordinates $\lambda\in\Delta^{p}$ as the \emph{barycentric coding model} (BCM).

\begin{figure}[t!]
    \centering
    \includegraphics[width=\linewidth]{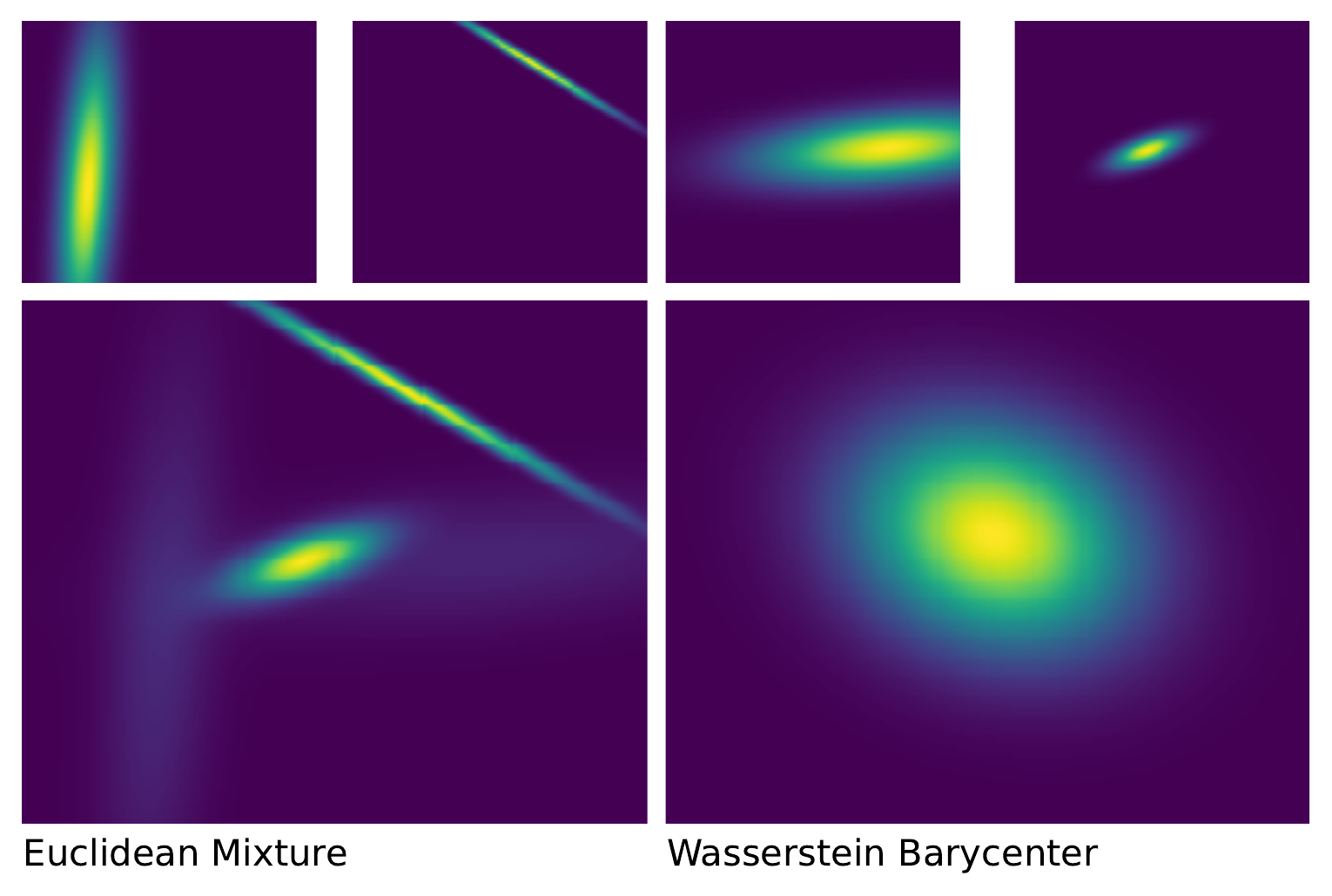}
    \caption{\emph{Top four}: reference Gaussians. \emph{Bottom left}: linear mixture with $\lambda=(\frac{1}{4},\frac{1}{4},\frac{1}{4},\frac{1}{4} )$. \emph{Bottom right}: Wasserstein barycenter with same $\lambda$. (Intensities adjusted to show detail.)}
    \label{fig:bc_vs_lin}
\end{figure}

\vspace{10pt}
\textbf{Main Contributions:} This paper addresses the analysis problem (\ref{eq:analysis}) as a measure estimation problem under the BCM. For this general problem, we make the following geometrical, statistical, and computational contributions:

\begin{enumerate}[(1)]
\setlength \itemsep{-3pt}
    \item When all the measures are directly observed, Theorem \ref{thm:main} provides a \emph{geometric} characterization of the solution to \eqref{eq:analysis} as corresponding to a solution to a quadratic program.  Our analysis leverages the Riemannian geometry of $\mathcal{W}_{2}$ space and establishes that the quadratic program is determined by the angles between the optimal displacement maps between the measures.  When specialized to the case of Gaussian measures, Corollary \ref{cor:gaussian} provides a closed-form characterization of the solution to the analysis problem.
    
    \item When the measures are indirectly observed via i.i.d. samples, we provide an algorithm for estimating the barycentric coordinates $\lambda$ of a given measure with respect to the reference measures.  Theorem \ref{thm:convergence} gives precise rates of convergence of the angles between the optimal displacement map by a computationally tractable estimator that exploits entropic regularization, thereby establishing the statistical consistency of our approach in Corollary \ref{cor:consistency}.
    
    \item We showcase the utility of the BCM and the proposed algorithm on several data analysis problems: covariance estimation for Gaussian measures; image inpainting and denoising; and document classification from limited labeled data. We show that the BCM provides a simple yet effective approach that outperforms competing methods in the regimes considered.\footnote{Code to reproduce results is available at \url{https://github.com/MattWerenski/BCM}} 
\end{enumerate}

\subsection{Related Work}

The synthesis problem (\ref{eq:barycenter}) is well-studied.  \citep{agueh2011barycenters} propose $W_{2}$ barycenters for a finite set of reference measures $\{\mu_{i}\}_{i=1}^{p}$ and prove their existence and uniqueness in a very general setting.  \citep{le2017existence} extend this to barycenters constructed from infinitely many reference measures. Statistical rates of estimation in the setting of i.i.d. samples from the measures $\{\mu_{i}\}_{i=1}^{p}$ are established in \cite{kroshnin2019complexity}.  Building on fundamental advances in estimating transport maps between measures using entropic regularization \cite{cuturi2013sinkhorn}, fast estimation methods for solving the synthesis problem have been developed for entroptically regularized barycenters \cite{cuturi2014fast}.  Approaches to computing a barycenter based on alternating direction of multipliers on the dual problem \cite{yang2021fast} and parallelization schemes \cite{claici2018stochastic} have also been considered.

The literature on the analysis problem (\ref{eq:analysis}) is much sparser.  To the best of our knowledge, our geometrical characterization of solutions to the analysis problem (Theorem \ref{thm:main}) and sample complexity results for estimating the optimal coordinates $\lambda$ (Theorem \ref{thm:convergence}, Corollary \ref{cor:consistency}) are the first of their kind.  Note that our approach extends \citep{bonneel2016coordinates, schmitz2018wasserstein}---which consider measure estimation under the BCM but only for measures with finite support---in three fundamental ways.  First, we enable the application of the BCM to absolutely continuous measures.  Second, we prove precise theoretical characterizations of the solutions to (\ref{eq:analysis}).  Third, we avoid the expensive procedure of differentiating through an iterative algorithm as required in \citep{bonneel2016coordinates} and instead provide the direct Algorithm \ref{alg1} that enjoys statistical consistency (Corollary \ref{cor:consistency}).  At the level of technical analysis, our Theorem \ref{thm:main} relies on an alternative characterization of $W_{2}$ barycenters that leverages known conditions under which Karcher means are guaranteed to be $W_{2}$ barycenters  \citep{panaretos2020invitation}.  

For the special case when all the measures are zero-mean Gaussians, the \emph{synthesis} problem is again well-studied, with a fixed-point algorithm for computing the barycenter given in \citep{alvarez2016fixed} that is shown to have a linear convergence rate in \citep{chewi2020gradient, altschuler2021averaging}.  The corresponding \emph{analysis} problem for the specific case of Gaussian measures has only been recently studied in \cite{musolas2021geodesically} but limited to the case when $p=2$. When specialized to this case, our results extend the model considered in \cite{musolas2021geodesically} to the $p \geq 3$ case and characterize solutions to barycenter parameterized covariance estimation.

In Section \ref{sec:emp_res}, we use the proposed BCM framework to model and address several problems in image processing and natural language processing. In this context our work is related to a long lineage of methods in signal processing and machine learning that use coefficients in a basis or a dictionary for data modeling and processing \citep{mallat1999wavelet, donoho2006compressed, tovsic2011dictionary}.  Our approach of representing a general measure in terms of its barycentric coordinates with respect to a fixed set of reference measures is a novel approach in this tradition, and our results in Section \ref{sec:emp_res} suggest its utility for data modeled as probability measures.

\section{Convex Optimization for the Analysis Problem}\label{sec:Theory}

To solve the analysis problem (\ref{eq:analysis}), we propose to solve a convex optimization problem based on an alternative characterization of $W_{2}$-barycenters as minimizers of a certain functional. We leverage this alternative characterization to show in Theorem \ref{thm:main} that, under mild assumptions on $\{\mu_{i}\}_{i=1}^{p}$, one can use the optimal value of a convex program to both check if $\mu_{0}\in\bary(\{\mu_{i}\})_{i=1}^{p}$ and if so, recover its barycentric coordinate. 
Before proceeding, we must introduce some background on the Riemannian geometry of the $\mathcal{W}_{2}$ space.  The geodesic curve from $\nu_0$ to $\nu_1$ in $\mathcal{W}_{2}$ is given by $[\nu_t]_{t=0}^1$ where $\nu_t = (tT + (1-t)\id)\#\nu_0$ where $T$ is the optimal transport map from $\nu_0$ to $\nu_1$ in (\ref{eq:wass_dist}) \cite{santambrogio2015optimal}.  These facts together with Brenier's Theorem (see Supplementary Material Section \ref{SM:PreciseStatements}) motivate the definition of the tangent space at $\nu$: \begin{align*}
    &T_\nu\Pac \\
\triangleq&\overline{\{\beta (\nabla \varphi - \id): \beta > 0, \varphi \in C_c^\infty(\mathbb{R}^d), \varphi \text{ convex}\}},
\end{align*} where the closure is in $L^{2}(\nu)$ \citep{ambrosio2005gradient}.  For $u,v \in T_\nu\Pac$, the inner product is defined by $\langle u, v \rangle_\nu \triangleq \int_{\mathbb{R}^d} \langle u(x), v(x) \rangle d\nu(x),$
which also gives the standard norm of $u \in T_\nu\Pac$, $||u||_\nu \triangleq \sqrt{\langle u, u\rangle_{\nu}}$.

\subsection{Alternative Characterization of Barycenters} 

Given the  reference measures $\{\mu_i\}_{i=1}^p$ and barycentric coordinates $\lambda \in \Delta^p$, the \emph{variance functional} is $G_{\lambda} : \Pac \rightarrow \mathbb{R}$ defined by $$\displaystyle G_{\lambda}(\nu) \triangleq \sum_{i=1}^p \frac{\lambda_i}{2} W_2^2(\nu, \mu_i),$$ so that the barycenter for $\{\mu_i\}_{i=1}^p$ and $\lambda \in \Delta^p$ is
$$\nu_{\lambda} =\displaystyle\argmin_{\nu \in \Pac} G_{\lambda}(\nu).$$

In many unconstrained optimization problems, the most direct way of checking if a point is optimal is to use first-order optimality conditions: compute the gradient at the point and check if it is zero. Without further assumptions this is a necessary, but not sufficient, condition for optimality. Applying this to a general variance functional, one arrives at the definition of a \emph{Karcher mean} \citep{karcher1977riemannian}.
\begin{definition}
    Let $(X,\rho)$ be a metric space furnished with a tangent space at every $x \in X$. For a collection of points $\{x_i\}_{i=1}^p \subset X$ and coordinates $\lambda \in \Delta^p$, a point $y\in X$ is said to be a \emph{Karcher mean} of the set for $\lambda$ if $G_{\lambda} = \sum_{i=1}^p \lambda_i \rho^2(\cdot, x_i)$ is differentiable at $y$ and $||\nabla G_{\lambda}(y)||_y = 0$ where $\nabla G_{\lambda}(y)$ is understood as an element of the tangent space at $y$.
\end{definition}

\begin{figure}
    \centering
    \includegraphics[width=0.35\textwidth]{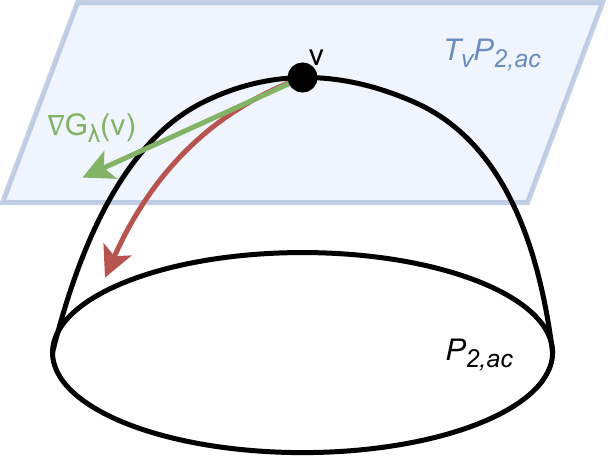}
    \caption{Visualization of the tangent space $T_{\nu}\Pac$ and $\nabla G_{\lambda}$. The red arc is the path along which $G_{\lambda}$ is maximized, starting at $\nu$ and the green arrow points in that direction in the tangent space at $\nu$.}
    \label{fig:tangent_space}
\end{figure}

Any barycenter $y$ is a global minimizer of $G_{\lambda}$. Using first order optimality conditions, this implies $||\nabla G_{\lambda}(y)||_y = 0$, and therefore $y$ must be a Karcher mean. Special care must be made to ensure that being a Karcher mean is sufficient to be a barycenter. When it is sufficient, it immediately gives an alternative test for $y$ to be a barycenter for $\lambda$: simply check if $||\nabla G_{\lambda}(y)||_y = 0$.

An example of a Karcher mean which is not a barycenter is given in the Supplementary Material Section \ref{sec:karcher_not_bc}. To ensure that being a Karcher mean is sufficient to be a barycenter we require the following assumptions:

\textbf{A1:}\,\,The measures $\{\mu_i\}_{i=0}^p$ are absolutely continuous and supported on either all of $\mathbb{R}^d$ or a bounded open convex subset. Call this shared support set $\Omega$.

\textbf{A2:}\,\,The measures $\{\mu_i\}_{i=0}^p$ have respective densities $\{g_i\}_{i=0}^p$ which are bounded above and $g_1,...,g_p$ are strictly positive on $\Omega$.

\textbf{A3:}\,\,If $\Omega = \mathbb{R}^d$ then $\{g_i\}_{i=0}^p$ are locally H\"older continuous. Otherwise $\{g_i\}_{i=0}^p$ are bounded away from zero on $\Omega$.

Under these assumptions if $\mu_0$ is a Karcher mean with coordinates $\lambda$, then $\mu_0 = \nu_{\lambda}$ \citep{panaretos2020invitation}; see the Supplementary Material Section \ref{SM:PreciseStatements} for a precise statement.

This characterization of barycenters can be much easier to work with and the assumptions required are mild in practice.  To accompany the notion of a Karcher mean, the \emph{Fr\'{e}chet derivative} \citep{frechet1948elements} of the variance functional $G_{\lambda}$ is given by
\begin{equation} \label{eq:grad}
    \nabla G_{\lambda}(\nu) = -\sum_{i=1}^p \lambda_i \left ( T_i - \id \right ),
\end{equation} where $T_i$ is the optimal transport map from $\nu$ to $\mu_i$ as in (\ref{eq:wass_dist}) \citep{ambrosio2005gradient, panaretos2020invitation}.

We now have all the tools to begin building up to Theorem \ref{thm:main}. The first step is to derive a formula for $||\nabla G_{\lambda}(\mu_0)||_{\mu_0}^2$ in terms of $\lambda$ and show its convexity. 
\\
\\
\begin{proposition}\label{prop:Prop1}
    Let $\{\mu_i\}_{i=0}^p \subset \Pac$. Then 
    $||\nabla G_{\lambda}(\mu_0)||_{\mu_0}^2 = \lambda^TA\lambda$ where $A \in \mathbb{R}^{p \times p}$ is given by 
    \begin{equation} \label{eq:A_ij}
        A_{ij} = \int_{\mathbb{R}^d} \langle T_i(x) - \id(x), T_j(x) - \id(x) \rangle d\mu_0(x)
    \end{equation}
    and $T_i$ is the optimal transport map from $\mu_0$ to $\mu_i$ as in (\ref{eq:wass_dist}). Furthermore, this is a convex function in $\lambda$.
\end{proposition}

The proof is given in the Supplementary Material Section \ref{SM:Prop1Proof}.  The expression $T_i(x) - \id(x)$ is the displacement of the point $x$ when being transported from $\mu_0$ to $\mu_i$. Under this interpretation, $\langle T_i(x) - \id, T_j(x) - \id \rangle$ can be thought of as the angle between the displacement associated to the optimal transport map between $\mu_0$ to $\mu_i$ with that of $\mu_0$ to $\mu_j$. Integrating this quantity with respect to $\mu_0$ therefore quantifies the average angle between displacements and if the optimal displacement maps move in similar directions.

We now use this functional form as part of a convex program and show that it can be used to check if a $\mu_0$ is a Karcher mean. Our first main result ties Propositions \ref{prop:Prop1} to the conditions for equivalence between barycenters and Karcher means.
\begin{theorem} \label{thm:main}
Let $\{\mu_{i}\}_{i=0}^{p}$ satisfy \textbf{A1}-\textbf{A3}. Then $\mu_0\in \bary(\{\mu_{i}\}_{i=1}^{p})$ if and only if 
\[\min_{\lambda\in\Delta^{p}} \lambda^TA\lambda = 0\]
where $A \in \mathbb{R}^{p \times p}$ is given by (\ref{eq:A_ij}). Furthermore, if the minimum value is 0 and $\lambda_*$ is an optimal argument, then $\mu_0 = \nu_{\lambda_*}$.
\end{theorem}

The proof is given in the Supplementary Material Section \ref{SM:Cor1Proof}. Theorem \ref{thm:main} gives a geometric characterization of the solution to (\ref{eq:analysis}) which relies on the evaluation of inner-products in the tangent space $T_{\mu_0}\Pac$ (noting that $A_{ij} = \langle T_i - \id, T_j - \id \rangle_{\mu_0}$) and solving a constrained quadratic program.  Furthermore, this result suggests a method for finding an approximator of the $\mu_0$ in the set of barycenters: simply minimize the objective and use $\lambda_*$, even if the minimum value is not zero. 

Note that if $A$ has rank less than $p-1$ there may be multiple minimizers which correspond to redundancies in the set $\bary(\{\mu_i\}_{i=1}^{p})$; see Supplementary Material Section \ref{sec:rank_def}.

\subsection{Theorem \ref{thm:main} For Gaussians}\label{subsec:SpecialCases}

In the Gaussian case there are formulas for both the optimal transport maps and the integration involved in $A_{ij}$. Let $\mathbb{S}_{++}^d$ denote the set of symmetric positive definite $d \times d$ matrices. Let $\mu_i = \mathcal{N}(0, S_i)$ with $S_i \in \mathbb{S}_{++}^d$ for $i=0,\dots,p$. Then the optimal transport map from $\mu_0$ to $\mu_i$ is given by $T_i(x) = C_ix$ where
\begin{equation} \label{eq:map_mat}
    C_i = S_0^{-1/2}\left ( S_0^{1/2} S_i S_0^{1/2} \right )^{1/2} S_0^{-1/2}
\end{equation}
\cite{agueh2011barycenters}. Furthermore, \textbf{A1}-\textbf{A3} are satisfied for non-degenerate Gaussians. Combining Theorem \ref{thm:main} with (\ref{eq:map_mat} gives the following.
\begin{corollary}
     \label{cor:gaussian}
    For $i=1,\dots,p$, let $\mu_i = \mathcal{N}(0,S_i)$ with $S_i\in\mathbb{S}_{++}^d$. Then $\mu_0$ is a barycenter if and only if $\mu_0 = \mathcal{N}(0,S_0)$ for some $S_0\in\mathbb{S}_{++}^d$, and the convex program of Theorem \ref{thm:main} has minimum value 0, where the matrix $A$ is given by $A_{ij} = \Tr\left ( (C_i - I)(C_j - I)S_0 \right ).$
    Furthermore, if the minimum value is zero and $\lambda_*$ is a optimal argument, then $\mu_0 = \nu_{\lambda_*}$
\end{corollary}

The proof is deferred to the Supplementary Material Section \ref{SM:Cor1Proof}.  We note that in the Gaussian case, one can use $||\nabla G_{\lambda}(\mu_0)||_{\mu_0}$ to upper bound $W_2(\mu_0, \nu_{\lambda})$ \citep{chewi2020gradient, altschuler2021averaging}. This further justifies the choice of optimizing $||\nabla G_{\lambda}(\mu_0)||_{\mu_0}^2$, since it can be seen as a convex, sharp upper bound on the distance to the set of barycenters.

\subsection{Projections of Compatible Measures}

If one imposes extra structure on the set $\{\mu_i\}_{i=0}^p$, then it is possible to substantially strengthen Theorem \ref{thm:main}. 

\begin{definition}
    A family of measures $\{\mu_{i}\}_{i=0}^{p}$ is said to be \emph{compatible} if for any $i,j,k \in \{0,1,...,p\}$ the optimal transport map $T_i^k$ from $\mu_i$ to $\mu_k$ can can be expressed as 
    $$T_i^k = T_j^k \circ T_i^j$$
    where $T_i^j,T_j^k$ are the optimal transport maps from $\mu_i$ to $\mu_j$ and $\mu_j$ to $\mu_k$ respectively.
\end{definition}
Note that whenever these maps exist, the map $ T_j^k \circ T_i^j$ is a valid map from $\mu_i$ to $\mu_k$, but it need not be optimal.  Under the condition of compatibility, we have the following result.
\begin{theorem} \label{thm:projection}
    Suppose that $\{\mu_i\}_{i=0}^{p}$ are compatible and let $\lambda_*$ be the minimizer in Theorem \ref{thm:main}. Then
    $$\lambda_* = \argmin_{\lambda \in \Delta^p} W_2^2(\mu_0, \nu_\lambda).$$
\end{theorem}
The proof is deferred to the Supplementary Material Section \ref{sec:projection}. In other words, when the measures are compatible one can solve (\ref{eq:analysis}) and find the \emph{exact projection} of $\mu_0$ onto $\bary(\{\mu_i\}_{i=1}^p).$

There are several known conditions under which the family of measures is compatible. These include any set of continuous measures on $\mathbb{R}$, any family of Gaussians with simultaneously diagonalizable covariance matrices, families of radial transformations of a base measure, or families with a "common dependence structure''; see \cite{panaretos2020invitation} Section 2.3 for further details. 

\section{Finite Sample Analysis and Rates of Convergence} \label{sec:SampleSetting}

The results of the previous section are all reliant upon having exact knowledge of both the underlying measure $\mu_0$ as well as the optimal transport maps $T_i$. In practice, outside of specific parametric families (e.g., Gaussians), it is rare that one will have either of these available. 

Much more common is the setting where the $\mu_i$ are accessed through i.i.d. sampling. In this setting we need to estimate the maps $T_i$, as well as $\mu_0$ and use these estimates to compute an $\hat{A}$ whose entries are $\hat{A}_{ij} = \langle \hat{T}_i - \id, \hat{T}_j - \id \rangle_{\hat{\mu}_0}$,
where $\hat{T}_i$ and $\hat{\mu_0}$ are our estimates of the transport maps and $\mu_0$ respectively. The quality of the estimate $\hat{A}$ will depend on the approximations used above, and ideally these would be accompanied by performance guarantees.  

\subsection{The Entropic Map Estimate}

Let $X_1,...,X_n \sim \mu_0$ and $Y_1,...,Y_n \sim \mu_1$ be i.i.d. samples. In \cite{pooladian2021entropic} the authors propose using the \textit{entropic map} to estimate the transport maps. For $\epsilon > 0$, this map is defined by 
\begin{equation} \label{eq:entropic_map}
    \hat{T}_\epsilon(x) \triangleq \frac{\frac{1}{n}\displaystyle\sum_{i=1}^n Y_i \exp\left (\frac{1}{\epsilon} (g_{\epsilon}(Y_i) - \frac{1}{2}||x - Y_i||_2^2) \right )}{\frac{1}{n}\displaystyle\sum_{i=1}^n \exp\left (\frac{1}{\epsilon} (g_{\epsilon}(Y_i) - \frac{1}{2}||x - Y_i||_2^2) \right )}
\end{equation}
where $g_\epsilon$ is defined to be the optimal variable in the dual entropic optimal transport  problem on the samples. See the Supplementary Material \ref{SM:Entropic_Map} for details on  $g_\epsilon$ and important properties of $\hat{T}_\epsilon$. Practically, the calculation of $g_\epsilon$ has been well studied and there exists fast algorithms for its computation \cite{peyre2020computational}. The estimate has also been shown to converge (with rates) to the optimal transport map in $L_2(\mu_0)$ under the appropriate technical conditions. 

To accompany this we also have the following result, stated informally below. In this result $a \lesssim b$ means that there exists a constant $C > 0$ such that $a \leq Cb$. We allow this constant $C$ to depend on the size of the support of the measures, upper and lower bounds on their densities, and the regularity properties of the optimal transport maps $T_i,T_j$, but it is importantly independent of the sample size $n$.
\begin{theorem}(\textbf{Informal}) \label{thm:convergence}
Let $i,j \in \{1,...,p\}$ and suppose that $\mu_i,\mu_j,\mu_0$ are supported on bounded domains and that the maps $T_i$ and $T_j$ are sufficiently regular. Let $X_1,...,X_{2n} \sim \mu_0, Y_1,...,Y_n \sim \mu_i, Z_1,...,Z_n \sim \mu_j$. For an appropriately chosen $\epsilon$, let $\hat{T}_i$ and $\hat{T}_j$ be the entropic maps computed using $\{X_{i}\}_{i=1}^{n}, \{Y_{i}\}_{i=1}^{n}, \{Z_{i}\}_{i=1}^{n}$. Then we have 
    \begin{equation}
    \begin{split}
        & \mathbb{E}\left [ \left | A_{ij} - \frac{1}{n} \sum_{k=n+1}^{2n} \langle \hat{T}_i(X_k) - X_k, \hat{T}_j(X_k) - X_k \rangle \right | \right ] \\
        &\hspace{1.5cm} \lesssim \frac{1}{\sqrt{n}} + n^{-\frac{\alpha + 1}{4(d' + \alpha + 1)}}\sqrt{\log n}
    \end{split}
    \end{equation}
    where  $d' = 2\lceil d/2 \rceil$, and $\alpha \leq 3$ depends on the regularity of optimal transport maps.
\end{theorem}

A precise statement and proof of this theorem is given in the Supplementary Material Section \ref{SM:Thm2Proof}. Theorem \ref{thm:convergence} tells us that in expectation, by using the entropic map we can accurately approximate the entries of the matrix $A_{ij}$, and do so in a numerically feasible way. We chose the entropic map as it is computationally tractable and has competitive convergence rates when compared to the intractable but optimal estimator of \cite{hutter2021minimax}.  We expect similar analysis holds for other recently proposed estimators \cite{deb2021rates, gunsilius_2021, manole2021plugin, muzellec2021near} under different assumptions.

Theorem \ref{thm:convergence} can be combined with a matrix perturbation analysis to demonstrate the consistency of the estimated barycentric coordinates $\hatlambda$ from Algorithm \ref{alg:estimate_lambda} for the true ones $\lambda_{*}$ when $\mu_{0}\in \bary(\{\mu_{i}\}_{i=1}^{p})$, that is when $\lambda_{*}^{T}A\lambda_{*}=0$. This is stated more precisely as follows. 

\begin{corollary}\label{cor:consistency} Let $\hatlambda$ be the random estimate obtained from Algorithm \ref{alg:estimate_lambda}.  Suppose that $A$ has an eigenvalue of 0 with multiplicity 1 and that $\lambda_{*}\in\Delta^{p}$ realizes $\lambda_{*}^{T}A\lambda_{*}=0$.  Then under the assumptions of Theorem \ref{thm:convergence}, \[\E[\|\hatlambda -\lambda_{*}\|_{2}^{2}]\lesssim \frac{1}{\sqrt{n}} + n^{-\frac{\alpha + 1}{4(d' + \alpha + 1)}}\sqrt{\log n}.\]

\end{corollary}

\begin{algorithm}[ht!]
\caption{\label{alg1}Estimate $\lambda$} \label{alg:estimate_lambda}
\begin{algorithmic}
\STATE {\bfseries Input:} i.i.d. samples $\{X_1,...,X_{2n}\} \sim \mu_0, \{\{Y_1^i,...,Y_n^i\} \sim \mu_i: i=1,...,p\},$ regularization parameter $\epsilon > 0$.
\FOR{$i = 1,...,p$} 
    \STATE Set $M^i \in \mathbb{R}^{n\times n}$ with $M^i_{jk} = \frac{1}{2}||X_j - Y_k^i||_2^2$.
    \STATE Solve for $g^i$ as the optimal $g$ in
    \begin{equation*} \begin{split}
        &\max_{f,g \in \mathbb{R}^n} \frac{1}{n}\sum_{j=1}^n f_j+ \frac{1}{n}\sum_{k=1}^n g_k \\ 
        &\hspace{1.5cm}- \frac{\epsilon}{n^2}\sum_{j,k}^n \exp \left((f_j + g_k - M_{jk}^i)/\epsilon \right).
    \end{split}
    \end{equation*}
    \STATE Define $\hat{T}_i$ through \ref{eq:entropic_map} with $g_\epsilon = g_i$ and $\{Y_1^i,...,Y_n^i\}$. 
\ENDFOR
\STATE Set $\hat{A} \in \mathbb{R}^{p \times p}$ to be the matrix with entries
\begin{equation*}
    \hat{A}_{ij} = \frac{1}{n}\sum_{k=n+1}^{2n} \langle \hat{T}_i(X_k) - X_k, \hat{T}_j(X_k) - X_k \rangle.
\end{equation*}
\STATE \textbf{Return} $\hat{\lambda}=\displaystyle\argmin_{\lambda \in \Delta^p} \lambda^T\hat{A}\lambda$.
\end{algorithmic}
\end{algorithm}

The proof is deferred to the Supplementary Material Section \ref{SM:Cor2Proof}.  Note that the implicit constant in the inequality of Corollary \ref{cor:consistency} depends on $p$ in addition to the dependencies of listed above.  Corollary \ref{cor:consistency} ensures that in the large sample limit, the entropic regularization parameter may be chosen to guarantee precise estimation of the true barycentric coordinates.  Note that the rate of convergence in Corollary \ref{cor:consistency} depends crucially on the smoothness of the transport maps between the reference measures ($\alpha$) and the dimensionality of the space in which they are supported ($d'$).

In the next section we illustrate the utility of the BCM  and propose novel approaches based on Theorem \ref{thm:main} and Algorithm \ref{alg:estimate_lambda} for several applications. 

\section{Applications \label{sec:emp_res}}

\subsection{Barycenter Parameterized Covariance Estimation}

Extending the set-up in \cite{musolas2021geodesically}, we first consider the case where $\mu_i = \mathcal{N}(0,S_i), S_i \in \mathbb{S}^d_{++}$  for $i=0,\dots,p$.  We will use $S_i$ both as a matrix and to refer to $\mu_i$; similarly, let $S_{\lambda}$ denote $\nu_{\lambda}$. Let $\{x_i\}_{i=1}^n$ be i.i.d. samples with empirical covariance $\hat{S}$. In this setting, we can use the formulas laid out in Corollary \ref{cor:gaussian}.

For comparison, we also consider maximum likelihood estimation (MLE) for the parameter $\lambda$ in the BCM. Let $P_{\lambda}$ denote the probability density function of $S_{\lambda}$. The MLE is $\argmax_{\lambda \in \Delta^p} \prod_{i=1}^n P_{\lambda}(x_i),$ which is equivalent to minimizing the KL-divergence $D_{KL}(\hat{S},S_{\lambda})$. However, this problem may be non-convex and difficult to optimize. We solve it numerically using auto-differentiation \citep{bartholomew2000automatic} and use the coordinates recovered by this method, which may not be the true MLE.

\textbf{Experimental Setup:}  We consider the problem of estimating the covariance matrix $S_0$ from i.i.d. samples $\{x_i\}_{i=1}^{n} \sim \mathcal{N}(0,S_0)$ when $S_0$ is known to be a barycenter of $S_1,...,S_p$. To do so we use the following procedure:
\vspace{-5pt}
\begin{enumerate}
\setlength \itemsep{-4pt}
    \item \textbf{Choose $\lambda$ and $\{S_i\}_{i=1}^p$:} Select $\lambda \in \Delta^p$ and the reference measures $\{S_{i}\}_{i=1}^{p}$ either by hand or at random.
    \item \textbf{Generate $S_{\lambda}$:} Using Algorithm 1 of \citep{chewi2020gradient}, find the true barycenter, $S_{\lambda}$.
    \item \textbf{Sample:} Sample $\{x_i\}_{i=1}^n$ i.i.d. from $\mathcal{N}(0,S_{\lambda})$. Compute the empirical covariance
    $\hat{S} = \frac{1}{n} \sum_{i=1}^n x_ix_i^T.$
    \item \textbf{Estimate $\hat{\lambda}$:} Perform the optimization in Corollary \ref{cor:gaussian}, or use MLE with $S_0 = \hat{S}$ to recover an estimate $\hat{\lambda}$.
    \item \textbf{Compute $S_{\hat{\lambda}}$:} Again using Algorithm 1 of \citep{chewi2020gradient} find $S_{\hat{\lambda}}$ and use that as the estimate of $S_{\lambda}$.
\end{enumerate}
The results are in Figure \ref{fig:gaussian}. These show that using our approach (denoted \emph{gradient norm}) we can estimate the underlying covariance matrix much more accurately than the empirical covariance. Our optimization is competitive with the auto-differentiation approach to MLE, while being significantly more stable and much faster.  Indeed, computing all the estimates $\hat{\lambda}$ to generate Figure \ref{fig:gaussian} takes our method less than 1 second compared to over 3 hours and 40 minutes for MLE. For reference it takes on average roughly 0.01 seconds to perform empirical covaraiance estimation\footnote{All reported times are obtained using 20 Intel Xeon CPU E5-2660 V4 cores at 2.00 GHz.}. Further details are in the Supplementary Material Section \ref{SM:Covariance}.


\begin{figure}[ht!]
    \centering
    \includegraphics[width=0.93\linewidth]{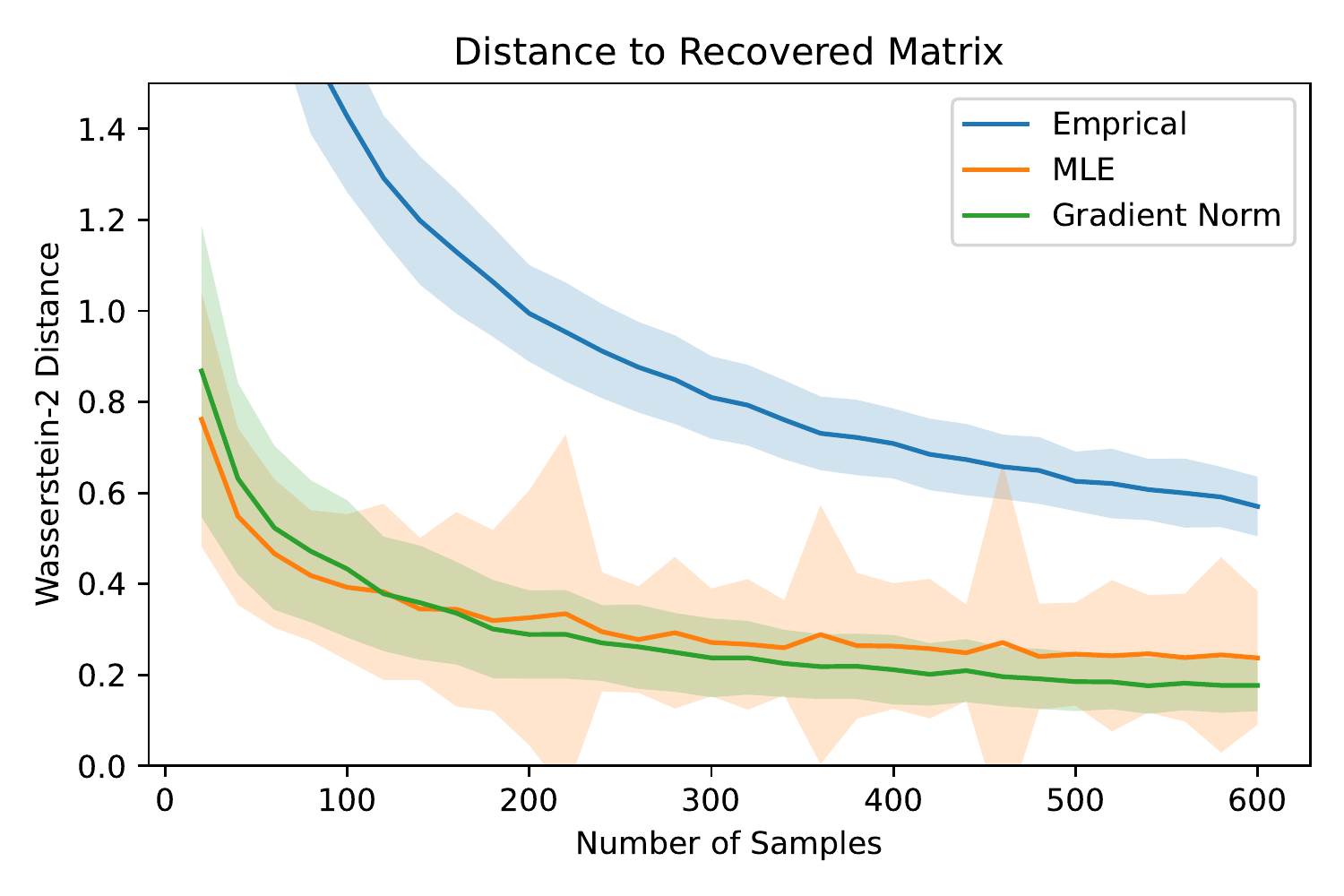}
    \caption{$W_{2}$ between true covariance and estimates.  We set $p = 6, d = 10$, $\lambda \sim \text{Unif}(\Delta^p)$, and $\mu_i \sim \text{Wish}(I_d) + 0.5 I_d$ with results averaged over 250 trials, and 1 standard deviation shaded.}
    \label{fig:gaussian}
\end{figure}

We now consider applications where it is necessary to directly estimate the optimal transport map from samples (unlike the previous case where we estimate it via sample-covariance estimators) and between measures with non-uniform weights on the observed support. To that end, we first modify Algorithm \ref{alg:estimate_lambda} for practical purposes.

\subsection{Adapting Algorithm \ref{alg:estimate_lambda} for Applications}
In practice, Algorithm \ref{alg:estimate_lambda} has a few mild requirements: that we use the same number of samples for each reference measure; that the weight on each sample is the same; and access to an extra set of points for evaluating the inner products. Algorithm \ref{alg:estimate_lambda_pc} relaxes these requirements and is in terms of matrix-vector operations, making it suitable for practical implementation. We note that if one uses Algorithm \ref{alg:estimate_lambda_pc} with the same constraints as in Algorithm \ref{alg:estimate_lambda}, then both produce the exact same $\hat{T}_i$ on the samples $X_1,...,X_n$ \cite{pooladian2021entropic}. The main difference between the two approaches is that the dual potentials in Algorithm \ref{alg:estimate_lambda} make an out-of-sample extension possible which may be of its own interest. 

\begin{algorithm} 
\caption{Estimate $\lambda$ on Point Clouds} \label{alg:estimate_lambda_pc}
\begin{algorithmic}
\STATE {\bfseries Input:} PMFs $p \in \mathbb{R}^{n_0}$, $q^i \in \mathbb{R}^{n_i}$, support matrices $X \in \mathbb{R}^{n_0 \times d}$,$Y^i \in \mathbb{R}^{n_i \times d}$, regularization parameter $\epsilon > 0$. 
\FOR{$i = 1,...,p$} 
    \STATE Set $M^i \in \mathbb{R}^{n_0 \times n_i}$ with $M^i_{jk} = ||X_j - Y_k^i||_2^2$.
    \STATE Solve for the entropic assignment matrix $\pi^i$ as the optimal matrix in
    \begin{equation*}
        \min_{\substack{\pi \in \mathbb{R}^{n_0 \times n_i}_+ \\
        \pi\bm{1} = p \\
        \pi^T\bm{1} = q^i}} \sum_{j=1}^{n_0}\sum_{k=1}^{n_i} M^i_{jk}\pi_{jk} + \epsilon \pi_{jk}\log \pi_{jk}.
    \end{equation*}
    \STATE Compute the approximate transport matrix $\hat{T}_i \in \mathbb{R}^{n_0 \times d}$ as
    \begin{equation*}
        \hat{T}_i = \text{diag}(1/p)\pi^i Y^i. 
    \end{equation*}
\ENDFOR
\STATE Set $\hat{A} \in \mathbb{R}^{p \times p}$ to be the matrix with entries
\begin{equation*}
    \hat{A}_{ij} = \Tr\left(\text{diag}(p)(\hat{T}^i-X)(\hat{T}^j-X)^T \right).
\end{equation*}
\STATE \textbf{Return} $\hat{\lambda}=\displaystyle\argmin_{\lambda \in \Delta^p} \lambda^T\hat{A}\lambda$.
\end{algorithmic}
\end{algorithm}

\subsection{Image Inpainting and Denoising} 

We consider the problem of recovering an image in the presence of corruption, interpreting the image as a probability measure.  We consider two specific models of corruption: \textbf{a.} additive \emph{noise} and \textbf{b.} \emph{occlusion} of a portion of the image.  Our experimental procedure for both of these problems is outlined below taking as experimental data the MNIST dataset of $28\times28$ pixel images of hand-written digits \citep{lecun1998mnist}. Additional details are in Supplementary Material Section \ref{SM:MNIST}.  Note that after normalization, these can be treated as measures supported on a $28\times28$ grid.

\textbf{Experimental Setup:} 

\begin{figure}[t]
    \centering
    \includegraphics[width=\linewidth]{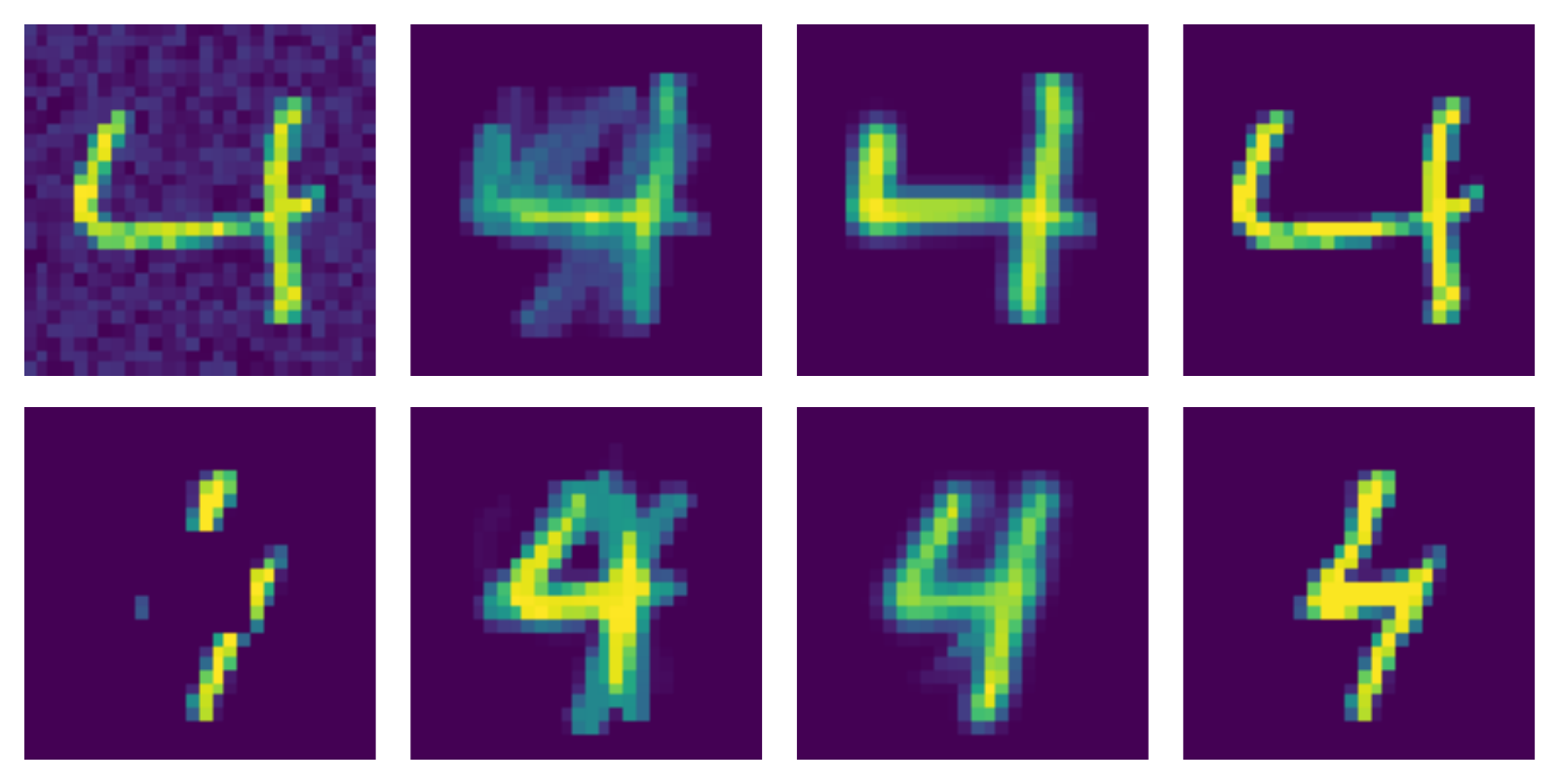}
    \caption{\emph{Left to right}: corrupted image; recovery by linear projection; recovery using $\hat{\lambda}$ from Algorithm \ref{alg:estimate_lambda_pc}; original image. \emph{Top}: white noise added to the image as in \textbf{a.} \emph{Bottom}: occlusion of the image as in \textbf{b.}  We see the linear reconstruction fails, while Algorithm \ref{alg:estimate_lambda_pc} recovers well, albeit with some blurring that can be interpreted as a consequence of entropic regularization.}
    \label{fig:MNIST_IO}
\end{figure}

\begin{enumerate}
\setlength \itemsep {-3pt}
    \item \textbf{Select $\mu_{0}$:}  Select a digit $\mu_0$ \textbf{a.} generate a white noise image $\zeta$. Set $\Tilde{\mu}_0$ to be $\Tilde{\mu_0} = (1 - \alpha)\mu_0 + \alpha \zeta$. \textbf{b.} Set $\Tilde{\mu}_0$ to be $\mu_0$ with the central $8 \times 8$ square removed, and renormalized.
    \item \textbf{Select $\{\mu_{i}\}_{i=1}^{p}$:} Select a set of images of the same digit as $\mu_0$ to serve as the reference measures. \textbf{a.} Let $\Tilde{\mu}_i = (1 - \alpha)\mu_0 + \alpha \mathbb{E}[\zeta]$. \textbf{b.} Let $\Tilde{\mu}_i$ be $\mu_i$ with the central $8 \times 8$ square removed, and renormalized
    \item \textbf{Estimate $A$, $\lambda$:} Use Algorithm \ref{alg:estimate_lambda_pc} to compute the approximate Gram matrix $\hat{A}$ using $\Tilde{\mu}_0, \Tilde{\mu_1}, ... \Tilde{\mu_p}$ and then compute estimated coordinates $\hat{\lambda}$.
    \item \textbf{Compute $\hat{\nu_{\lambda}}$:} Output $\hat{\mu}_0 = \nu_{\hat{\lambda}}$, where the barycenter is reconstructed from $\mu_1,...,\mu_p$.
\end{enumerate}

An illustration of the procedure is shown in Figure \ref{fig:MNIST_IO}.  We also show for comparison digits reconstructed using a linear method, in which the corrupted image is recovered by computing the Euclidean projection onto the convex hull of the reference measures (see the Supplement \ref{sec:linear_rec} for further details). To recover the barycenter we use the method of \cite{benamou2015iterative}. As a competitive baseline we compare to a well-known existing method for histogram regression \cite{bonneel2016coordinates} in this setting. Our results are summarized in Tables \ref{tab:mnist} and \ref{tab:mnist_comp}. We note that our method is over an order of magnitude faster than \cite{bonneel2016coordinates} on this dataset and achieves competitive results.  This is particularly remarkable because \cite{bonneel2016coordinates} is specifically adapted to measures with structured support such as grids and meshes, while our Algorithm \ref{alg:estimate_lambda_pc} is much more general. 
\begin{table*}[]
\centering
\begin{tabular}{c|c|c|c|}
\cline{2-4}
    & Bon., $\epsilon = 10$ & Alg. 2, $\epsilon = 10$ & Alg. 2, $\epsilon = 0$ \\ 
    \hline
    \multicolumn{1}{|c|}{Occlusion}  & 2.5101 (2228s)  & 2.5488 (3.371s) & 2.5287 (1.062s) \\ 
    \hline
    \multicolumn{1}{|c|}{Noise ($\alpha=0.5$)} & 2.4058 (2391s) & 2.6797 (86.32s) & 2.3787 (50.67s) \\ 
    \hline
\end{tabular}
\caption{Average quality of recovery of $\mu_0$ measured in $W_2^2(\mu_0, \hat{\mu}_0)$ when reconstructing 500 random 4's using a barycenter constructed from 10 random 4's, as well as run times of each method.  Algorithm \ref{alg:estimate_lambda_pc} affords comparably accurate reconstructions in orders-of-magnitude faster time than the state-of-the-art \cite{bonneel2016coordinates}\label{tab:mnist}.}
\end{table*}

\begin{table*}[]
    \centering
    \begin{tabular}{c|c|c|c|}
    \cline{2-4}
        & (Alg. 2, $\epsilon=10$) - (Bon., $\epsilon = 10$) & (Alg. 2, $\epsilon=0$) - (Bon., $\epsilon = 10$) & (Alg. 2, $\epsilon=10$) - (Alg. 2, $\epsilon=0$) \\
        \hline 
        \multicolumn{1}{|c|}{Occlusion}  &    $0.03872 \pm 0.3103$    &    $0.01868 \pm 0.2934$   &     $0.02004 \pm 0.2145$    \\
        \hline 
        \multicolumn{1}{|c|}{Additive}   &    $0.2739 \pm 0.3564$     &   $-0.02706 \pm 0.2558$   &     $0.3009 \pm 0.2486$  \\
        \hline
    \end{tabular}
    \caption{Average and standard deviation of the pairwise-difference in reconstruction quality measured in $W_2^2(\mu_0, \hat{\mu}_0)$ in the same setting as Table \ref{tab:mnist}. The average pairwise-difference is smaller than the standard deviation in all cases involving \cite{bonneel2016coordinates}, suggesting there is no statistically significant difference.  Our approach should thus be preferred to \cite{bonneel2016coordinates}, on the grounds that it is considerably more efficient while achieving essentially the same accuracy.}
    \label{tab:mnist_comp}
\end{table*}

\subsection{Document Classification}

Finally, we consider the task of identifying the topic of a document using its word embedding representation as the empirical distribution. We assume that there are $t$ topics and that we have $p$ reference documents about each topic. Each document can be represented as a high-dimensional empirical measure using a combination of a bag-of-words representation and a node2vec \cite{mikolov2013distributed} word-embedding, and is paired with topic label. We use the publicly available dataset provided by \cite{huang2016supervised}. The task is to use a small number of labelled documents from each topic to predict the topic of a test document.

We consider four predictors: \textbf{(1) 1-Nearest Neighbor (1NN)} which classifies using the topic of the $W_2^2$-nearest reference document; \textbf{(2) Minimum Average Distance} which selects the topic with reference documents on average $W_2^2$-closest to the test document; \textbf{(3) Minimum Barycenter Loss} which runs Algorithm \ref{alg:estimate_lambda_pc} using the references from each topic separately and selects the topic with the smallest loss in the quadratic form; \textbf{(4) Maximum Coordinate} which runs Algorithm \ref{alg:estimate_lambda_pc} using all the reference documents and chooses the topic which receives the most mass from $\hatlambda$. 

\begin{figure}[t!]
    \centering
    \includegraphics[width=\linewidth]{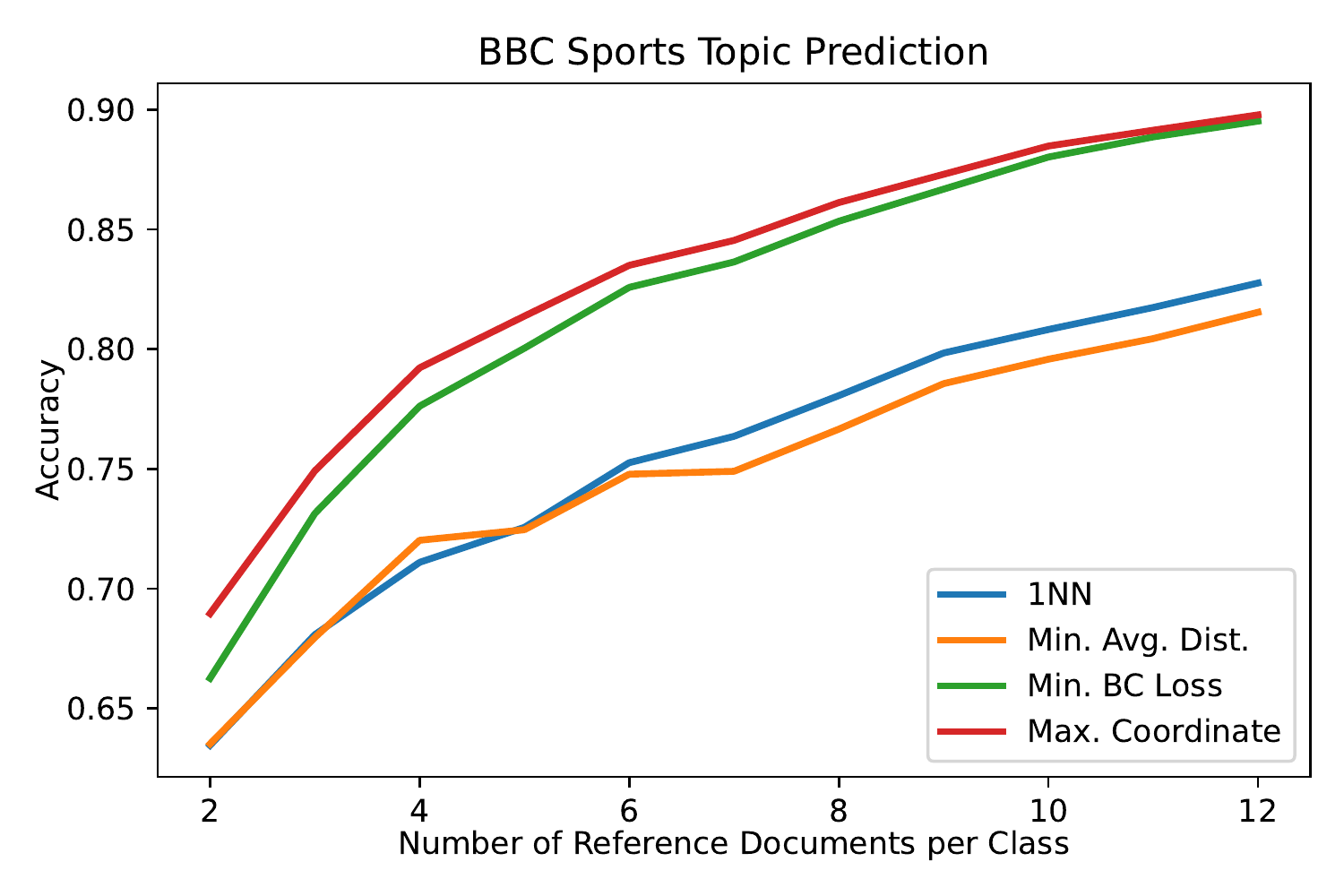}
    \includegraphics[width=\linewidth]{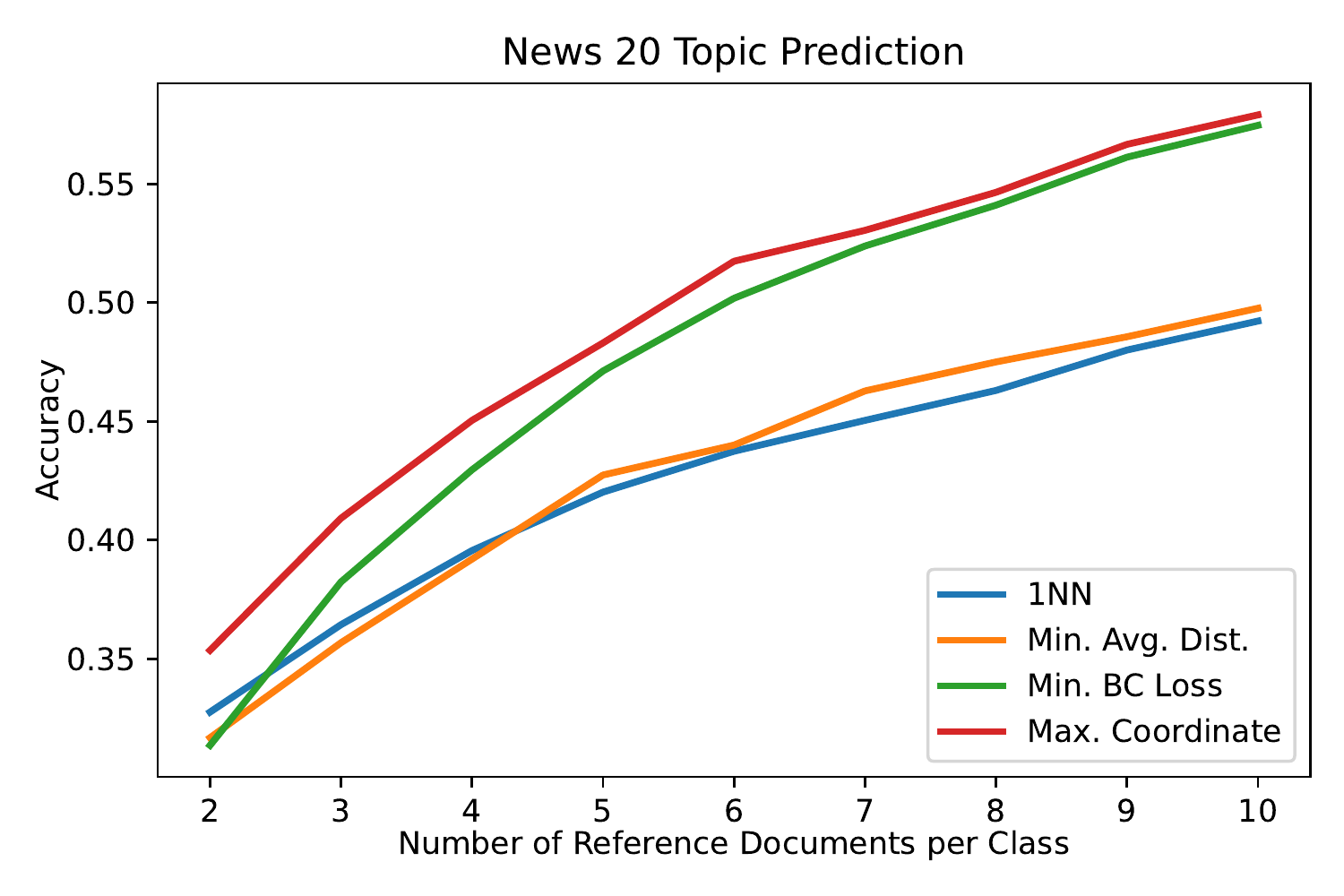}
    \caption{Document topic prediction accuracy as a function of number of reference documents in each class.  \emph{Top:}  BBC Sport dataset (5 classes). \emph{Bottom:} News20 dataset (20 classes).}
    \label{fig:NLP}
\end{figure}

Our findings on the BBCSport (5 topics) and News20 (20 topics) datasets are reported in Figure \ref{fig:NLP}.  To generate these figures we randomly sample $k$ reference documents per topic and a test set of 100 random documents. We apply the four methods listed above and compute the accuracy on the test set. This procedure is repeated 50 times for each choice of $k$ and the average accuracy is plotted. See Supplementary Material Section \ref{sec:NLP_figs} for details.

We see that even with a very small amount of training data, our BCM approaches perform well.  Importantly, the BCM is able to represent unseen documents using these small training sets, giving it a clear advantage over simply classifying based on which document class it is $W_2^2$-nearest (1NN) or $W_2^2$-nearest in an aggregate sense (Min. Avg. Dist.).  This suggests that even when $p$ is small, the BCM has high modeling capacity and that the learned barycentric coordinates encode important information.

\section{Conclusion}

We have proposed a new method for computing coordinates under the BCM, together with guarantees on its solution via a convex program (Theorem \ref{thm:main}) with closed-form coordinates in the Gaussian case (Corollary \ref{cor:gaussian}).  We further developed an algorithm for estimating the coordinates under the BCM when all measures are accessible only via i.i.d. samples (Algorithm \ref{alg1}) which enjoys a natural smoothness and dimension-dependent rate of convergence to the true parameters (Theorem \ref{thm:convergence}, Corollary \ref{cor:consistency}).  The BCM paradigm affords significant gains in runtime and robustness for barycenter parameterized Gaussian measure estimation, and provides an effective approach to image reconstruction and document classification when interpreting these data as measures.  The results in this paper suggest the efficiency and effectiveness of the BCM as a broadly viable modelling tool.

\textbf{Future Work and Open Questions:}  Section \ref{sec:SampleSetting} provides a consistency analysis for $A$ and $\lambda$ by estimating the transport maps between each measure.  But, all that is needed is an \emph{estimate on the inner product} between the displacement maps, which in principle could admit more sample efficient approaches that avoid explicitly estimating the maps.  

Section \ref{sec:emp_res} suggests that unseen data can be well-represented in the BCM for randomly chosen reference measures.  An interesting theoretical problem is to understand the representational capacity of the BCM using particular reference measures (e.g., random ones, optimally chosen ones, those in certain parametric families).  This naturally touches on the question of the smoothness of the synthesis map $\lambda \mapsto \nu_{\lambda}.$ 

Computationally, all of our real data experiments use randomly chosen reference measures.  It is of interest to develop efficient procedures as in \cite{schmitz2018wasserstein} for \emph{learning} reference measures that induce low reconstruction error.  Beyond that, one could regularize the analysis problem (\ref{eq:analysis}) to encourage sparse or otherwise structured coordinates as is well-studied in Euclidean settings \cite{aharon2006k}.

\section*{Acknowledgements}

MW is supported by NSF CCF-1553075. RJ is supported by NSF DRL 1931978.  SA acknowledges support by NSF CCF-1553075, NSF DRL 1931978, NSF EEC 1937057, and AFOSR FA9550-18-1-0465.  JMM acknowledges support from the Dreyfus Foundation and the NSF through grants DMS-1912737, DMS-1924513.  All authors acknowledge support through the Tufts TRIPODS Institute, supported by the NSF under grant CCF-1934553.  We acknowledge the reviewers for their helpful comments which improved the paper considerably.  

\bibliography{ref_camera_ready.bib}

\begin{thebibliography}{50}
\providecommand{\natexlab}[1]{#1}
\providecommand{\url}[1]{\texttt{#1}}
\expandafter\ifx\csname urlstyle\endcsname\relax
  \providecommand{\doi}[1]{doi: #1}\else
  \providecommand{\doi}{doi: \begingroup \urlstyle{rm}\Url}\fi

\bibitem[Agueh \& Carlier(2011)Agueh and Carlier]{agueh2011barycenters}
Agueh, M. and Carlier, G.
\newblock Barycenters in the {W}asserstein space.
\newblock \emph{SIAM Journal on Mathematical Analysis}, 43\penalty0
  (2):\penalty0 904--924, 2011.

\bibitem[Aharon et~al.(2006)Aharon, Elad, and Bruckstein]{aharon2006k}
Aharon, M., Elad, M., and Bruckstein, A.
\newblock K-{SVD}: An algorithm for designing overcomplete dictionaries for
  sparse representation.
\newblock \emph{IEEE Transactions on Signal Processing}, 54\penalty0
  (11):\penalty0 4311--4322, 2006.

\bibitem[Altschuler et~al.(2021)Altschuler, Chewi, Gerber, and
  Stromme]{altschuler2021averaging}
Altschuler, J.~M., Chewi, S., Gerber, P., and Stromme, A.~J.
\newblock Averaging on the {B}ures-{W}asserstein manifold: dimension-free
  convergence of gradient descent.
\newblock In \emph{Advances in Neural Information Processing Systems},
  volume~34, 2021.

\bibitem[{\'A}lvarez-Esteban et~al.(2016){\'A}lvarez-Esteban, Del~Barrio,
  Cuesta-Albertos, and Matr{\'a}n]{alvarez2016fixed}
{\'A}lvarez-Esteban, P.~C., Del~Barrio, E., Cuesta-Albertos, J., and
  Matr{\'a}n, C.
\newblock A fixed-point approach to barycenters in {W}asserstein space.
\newblock \emph{Journal of Mathematical Analysis and Applications},
  441\penalty0 (2):\penalty0 744--762, 2016.

\bibitem[Ambrosio et~al.(2005)Ambrosio, Gigli, and
  Savare]{ambrosio2005gradient}
Ambrosio, L., Gigli, N., and Savare, G.
\newblock \emph{Gradient Flows: In Metric Spaces and in the Space of
  Probability Measures}.
\newblock Lectures in Mathematics. ETH Z{\"u}rich. Birkh{\"a}user Basel, 2005.

\bibitem[Bartholomew-Biggs et~al.(2000)Bartholomew-Biggs, Brown, Christianson,
  and Dixon]{bartholomew2000automatic}
Bartholomew-Biggs, M., Brown, S., Christianson, B., and Dixon, L.
\newblock Automatic differentiation of algorithms.
\newblock \emph{Journal of Computational and Applied Mathematics},
  124:\penalty0 171--190, 12 2000.

\bibitem[Benamou et~al.(2015)Benamou, Carlier, Cuturi, Nenna, and
  Peyr{\'e}]{benamou2015iterative}
Benamou, J.-D., Carlier, G., Cuturi, M., Nenna, L., and Peyr{\'e}, G.
\newblock Iterative {B}regman projections for regularized transportation
  problems.
\newblock \emph{SIAM Journal on Scientific Computing}, 37\penalty0
  (2):\penalty0 A1111--A1138, 2015.

\bibitem[Bonneel et~al.(2015)Bonneel, Rabin, Peyr{\'e}, and
  Pfister]{bonneel2015sliced}
Bonneel, N., Rabin, J., Peyr{\'e}, G., and Pfister, H.
\newblock Sliced and {R}adon {W}asserstein barycenters of measures.
\newblock \emph{Journal of Mathematical Imaging and Vision}, 51\penalty0
  (1):\penalty0 22--45, 2015.

\bibitem[Bonneel et~al.(2016)Bonneel, Peyr{\'e}, and
  Cuturi]{bonneel2016coordinates}
Bonneel, N., Peyr{\'e}, G., and Cuturi, M.
\newblock {Wasserstein barycentric coordinates: histogram regression using
  optimal transport}.
\newblock \emph{{ACM Transactions on Graphics}}, 35\penalty0 (4):\penalty0
  71:1--71:10, 2016.

\bibitem[Brenier(1991)]{brenier1991polar}
Brenier, Y.
\newblock Polar factorization and monotone rearrangement of vector-valued
  functions.
\newblock \emph{Communications on Pure and Applied Mathematics}, 44\penalty0
  (4):\penalty0 375--417, 1991.

\bibitem[Cheng et~al.(2021)Cheng, Aeron, Hughes, and
  Miller]{cheng2021dynamical}
Cheng, K., Aeron, S., Hughes, M.~C., and Miller, E.
\newblock Dynamical {W}asserstein barycenters for time-series modeling.
\newblock In \emph{Advances in Neural Information Processing Systems},
  volume~34, 2021.

\bibitem[Chewi et~al.(2020)Chewi, Maunu, Rigollet, and
  Stromme]{chewi2020gradient}
Chewi, S., Maunu, T., Rigollet, P., and Stromme, A.~J.
\newblock Gradient descent algorithms for {B}ures-{W}asserstein barycenters.
\newblock In \emph{Conference on Learning Theory}, pp.\  1276--1304. PMLR,
  2020.

\bibitem[Chizat et~al.(2020)Chizat, Roussillon, L{\'e}ger, Vialard, and
  Peyr{\'e}]{chizat2020faster}
Chizat, L., Roussillon, P., L{\'e}ger, F., Vialard, F.-X., and Peyr{\'e}, G.
\newblock Faster {W}asserstein distance estimation with the {S}inkhorn
  divergence.
\newblock \emph{arXiv preprint arXiv:2006.08172}, 2020.

\bibitem[Claici et~al.(2018)Claici, Chien, and Solomon]{claici2018stochastic}
Claici, S., Chien, E., and Solomon, J.
\newblock Stochastic {W}asserstein barycenters.
\newblock In \emph{International Conference on Machine Learning}, pp.\
  999--1008. PMLR, 2018.

\bibitem[Colombo et~al.(2021)Colombo, Staerman, Clavel, and
  Piantanida]{colombo2021automatic}
Colombo, P., Staerman, G., Clavel, C., and Piantanida, P.
\newblock Automatic text evaluation through the lens of {W}asserstein
  barycenters.
\newblock \emph{arXiv preprint arXiv:2108.12463}, 2021.

\bibitem[Cuturi(2013)]{cuturi2013sinkhorn}
Cuturi, M.
\newblock Sinkhorn distances: Lightspeed computation of optimal transport.
\newblock \emph{Advances in Neural Information Processing Systems},
  26:\penalty0 2292--2300, 2013.

\bibitem[Cuturi \& Doucet(2014)Cuturi and Doucet]{cuturi2014fast}
Cuturi, M. and Doucet, A.
\newblock Fast computation of {W}asserstein barycenters.
\newblock In \emph{International Conference on Machine Learning}, pp.\
  685--693. PMLR, 2014.

\bibitem[Deb et~al.(2021)Deb, Ghosal, and Sen]{deb2021rates}
Deb, N., Ghosal, P., and Sen, B.
\newblock Rates of estimation of optimal transport maps using plug-in
  estimators via barycentric projections.
\newblock \emph{Advances in Neural Information Processing Systems}, 34, 2021.

\bibitem[Donoho(2006)]{donoho2006compressed}
Donoho, D.~L.
\newblock Compressed sensing.
\newblock \emph{IEEE Transactions on Information Theory}, 52\penalty0
  (4):\penalty0 1289--1306, 2006.

\bibitem[Fr{\'e}chet(1948)]{frechet1948elements}
Fr{\'e}chet, M.
\newblock Les {\'e}l{\'e}ments al{\'e}atoires de nature quelconque dans un
  espace distanci{\'e}.
\newblock \emph{Annales de l'institut {H}enri {P}oincar{\'e}}, 10\penalty0
  (4):\penalty0 215--310, 1948.

\bibitem[Genevay(2019)]{genevay2019entropy}
Genevay, A.
\newblock \emph{Entropy-regularized optimal transport for machine learning}.
\newblock PhD thesis, Paris Sciences et Lettres (ComUE), 2019.

\bibitem[Gunsilius(2021)]{gunsilius_2021}
Gunsilius, F.~F.
\newblock On the convergence rate of potentials of brenier maps.
\newblock \emph{Econometric Theory}, pp.\  1–37, 2021.

\bibitem[Huang et~al.(2016)Huang, Quo, Kusner, Sun, Weinberger, and
  Sha]{huang2016supervised}
Huang, G., Quo, C., Kusner, M.~J., Sun, Y., Weinberger, K.~Q., and Sha, F.
\newblock Supervised word mover's distance.
\newblock In \emph{Advances in Neural Information Processing Systems}, pp.\
  4869--4877, 2016.

\bibitem[H{\"u}tter \& Rigollet(2021)H{\"u}tter and
  Rigollet]{hutter2021minimax}
H{\"u}tter, J.-C. and Rigollet, P.
\newblock Minimax estimation of smooth optimal transport maps.
\newblock \emph{The Annals of Statistics}, 49\penalty0 (2):\penalty0
  1166--1194, 2021.

\bibitem[Karcher(1977)]{karcher1977riemannian}
Karcher, H.
\newblock Riemannian center of mass and mollifier smoothing.
\newblock \emph{Communications on Pure and Applied Mathematics}, 30\penalty0
  (5):\penalty0 509--541, 1977.

\bibitem[Kroshnin et~al.(2019)Kroshnin, Tupitsa, Dvinskikh, Dvurechensky,
  Gasnikov, and Uribe]{kroshnin2019complexity}
Kroshnin, A., Tupitsa, N., Dvinskikh, D., Dvurechensky, P., Gasnikov, A., and
  Uribe, C.
\newblock On the complexity of approximating {W}asserstein barycenters.
\newblock In \emph{International Conference on Machine Learning}, pp.\
  3530--3540. PMLR, 2019.

\bibitem[Le~Gouic \& Loubes(2017)Le~Gouic and Loubes]{le2017existence}
Le~Gouic, T. and Loubes, J.-M.
\newblock Existence and consistency of {W}asserstein barycenters.
\newblock \emph{Probability Theory and Related Fields}, 168\penalty0
  (3):\penalty0 901--917, 2017.

\bibitem[LeCun(1998)]{lecun1998mnist}
LeCun, Y.
\newblock The {MNIST} database of handwritten digits.
\newblock \emph{http://yann. lecun. com/exdb/mnist/}, 1998.

\bibitem[Mallat(1999)]{mallat1999wavelet}
Mallat, S.
\newblock \emph{A wavelet tour of signal processing}.
\newblock Elsevier, 1999.

\bibitem[Manole et~al.(2021)Manole, Balakrishnan, Niles-Weed, and
  Wasserman]{manole2021plugin}
Manole, T., Balakrishnan, S., Niles-Weed, J., and Wasserman, L.
\newblock Plugin estimation of smooth optimal transport maps.
\newblock \emph{arXiv preprint arXiv:2107.12364}, 2021.

\bibitem[McCann(1997)]{mccann1997convexity}
McCann, R.~J.
\newblock A convexity principle for interacting gases.
\newblock \emph{Advances in Mathematics}, 128\penalty0 (1):\penalty0 153--179,
  1997.

\bibitem[Mikolov et~al.(2013)Mikolov, Sutskever, Chen, Corrado, and
  Dean]{mikolov2013distributed}
Mikolov, T., Sutskever, I., Chen, K., Corrado, G.~S., and Dean, J.
\newblock Distributed representations of words and phrases and their
  compositionality.
\newblock In \emph{Advances in Neural Information Processing Systems}, pp.\
  3111--3119, 2013.

\bibitem[Montesuma \& Mboula(2021)Montesuma and
  Mboula]{montesuma2021wasserstein}
Montesuma, E.~F. and Mboula, F. M.~N.
\newblock Wasserstein barycenter for multi-source domain adaptation.
\newblock In \emph{IEEE/CVF Conference on Computer Vision and Pattern
  Recognition}, pp.\  16785--16793, 2021.

\bibitem[Murphy(2012)]{murphy2012machine}
Murphy, K.~P.
\newblock \emph{Machine learning: a probabilistic perspective}.
\newblock MIT press, 2012.

\bibitem[Musolas et~al.(2021)Musolas, Smith, and
  Marzouk]{musolas2021geodesically}
Musolas, A., Smith, S.~T., and Marzouk, Y.
\newblock Geodesically parameterized covariance estimation.
\newblock \emph{SIAM Journal on Matrix Analysis and Applications}, 42\penalty0
  (2):\penalty0 528--556, 2021.

\bibitem[Muzellec et~al.(2021)Muzellec, Vacher, Bach, Vialard, and
  Rudi]{muzellec2021near}
Muzellec, B., Vacher, A., Bach, F., Vialard, F.-X., and Rudi, A.
\newblock Near-optimal estimation of smooth transport maps with kernel
  sums-of-squares.
\newblock \emph{arXiv preprint arXiv:2112.01907}, 2021.

\bibitem[Panaretos \& Zemel(2020)Panaretos and Zemel]{panaretos2020invitation}
Panaretos, V.~M. and Zemel, Y.
\newblock \emph{An invitation to statistics in {W}asserstein space}.
\newblock Springer Nature, 2020.

\bibitem[Peyré \& Cuturi(2020)Peyré and Cuturi]{peyre2020computational}
Peyré, G. and Cuturi, M.
\newblock Computational optimal transport, 2020.

\bibitem[Pooladian \& Niles-Weed(2021)Pooladian and
  Niles-Weed]{pooladian2021entropic}
Pooladian, A.-A. and Niles-Weed, J.
\newblock Entropic estimation of optimal transport maps.
\newblock \emph{arXiv: 2109.12004}, 2021.

\bibitem[Popoviciu(1935)]{popoviciu1935equations}
Popoviciu, T.
\newblock Sur les {\'e}quations alg{\'e}briques ayant toutes leurs racines
  r{\'e}elles.
\newblock \emph{Mathematica}, 9:\penalty0 129--145, 1935.

\bibitem[Rabin et~al.(2011)Rabin, Peyr{\'e}, Delon, and
  Bernot]{rabin2011wasserstein}
Rabin, J., Peyr{\'e}, G., Delon, J., and Bernot, M.
\newblock Wasserstein barycenter and its application to texture mixing.
\newblock In \emph{International Conference on Scale Space and Variational
  Methods in Computer Vision}, pp.\  435--446. Springer, 2011.

\bibitem[Redko et~al.(2019)Redko, Courty, Flamary, and Tuia]{redko2019optimal}
Redko, I., Courty, N., Flamary, R., and Tuia, D.
\newblock Optimal transport for multi-source domain adaptation under target
  shift.
\newblock In \emph{International Conference on Artificial Intelligence and
  Statistics}, pp.\  849--858. PMLR, 2019.

\bibitem[Santambrogio(2015)]{santambrogio2015optimal}
Santambrogio, F.
\newblock \emph{Optimal Transport for Applied Mathematicians: Calculus of
  Variations, PDEs, and Modeling}.
\newblock Progress in Nonlinear Differential Equations and Their Applications.
  Springer International Publishing, 2015.

\bibitem[Schmitz et~al.(2018)Schmitz, Heitz, Bonneel, Ngole, Coeurjolly,
  Cuturi, Peyr{\'e}, and Starck]{schmitz2018wasserstein}
Schmitz, M.~A., Heitz, M., Bonneel, N., Ngole, F., Coeurjolly, D., Cuturi, M.,
  Peyr{\'e}, G., and Starck, J.-L.
\newblock Wasserstein dictionary learning: optimal transport-based unsupervised
  nonlinear dictionary learning.
\newblock \emph{SIAM Journal on Imaging Sciences}, 11\penalty0 (1):\penalty0
  643--678, 2018.

\bibitem[Schwerdtfeger(1961)]{schwerdtfeger1961introduction}
Schwerdtfeger, H.
\newblock \emph{Introduction to linear algebra and the theory of matrices}.
\newblock P. Noordhoff, 1961.

\bibitem[Singh et~al.(2020)Singh, Hug, Dieuleveut, and Jaggi]{singh2020context}
Singh, S.~P., Hug, A., Dieuleveut, A., and Jaggi, M.
\newblock Context mover’s distance \& barycenters: Optimal transport of
  contexts for building representations.
\newblock In \emph{International Conference on Artificial Intelligence and
  Statistics}, pp.\  3437--3449. PMLR, 2020.

\bibitem[To{\v{s}}i{\'c} \& Frossard(2011)To{\v{s}}i{\'c} and
  Frossard]{tovsic2011dictionary}
To{\v{s}}i{\'c}, I. and Frossard, P.
\newblock Dictionary learning.
\newblock \emph{IEEE Signal Processing Magazine}, 28\penalty0 (2):\penalty0
  27--38, 2011.

\bibitem[Villani(2003)]{villani2003topics}
Villani, C.
\newblock \emph{Topics in Optimal Transportation}.
\newblock Graduate studies in mathematics. American Mathematical Society, 2003.

\bibitem[Xu et~al.(2018)Xu, Wang, Liu, and Carin]{xu2018distilled}
Xu, H., Wang, W., Liu, W., and Carin, L.
\newblock Distilled {W}asserstein learning for word embedding and topic
  modeling.
\newblock In \emph{Advances in Neural Information Processing Systems}, pp.\
  1716--1725, 2018.

\bibitem[Yang et~al.(2021)Yang, Li, Sun, and Toh]{yang2021fast}
Yang, L., Li, J., Sun, D., and Toh, K.-C.
\newblock A fast globally linearly convergent algorithm for the computation of
  {W}asserstein barycenters.
\newblock \emph{Journal of Machine Learning Research}, 22\penalty0
  (21):\penalty0 1--37, 2021.

\end{thebibliography}
\bibliographystyle{icml2022}

\onecolumn
\setcounter{section}{0}

\begin{center}
    \textbf{\large{Supplementary Material: \\\emph{Measure Estimation in the Barycentric Coding Model}}}
\end{center}

\section{Precise Statements of Technical Results}
\label{SM:PreciseStatements}

For completeness, we replicate the theorem describing conditions under which Karcher means are guaranteed to be barycenters.
\begin{theorem}(\citep{panaretos2020invitation}, Theorem 3.1.15) \label{thm:karcher}
    Let $\lambda \in \Delta^p$ with $\lambda_{i} \neq 0$ for all $i$. Let $\Omega \subset \mathbb{R}^d$ be open and convex and let $\mu_1,\dots,\mu_p \in \Pac$ supported on $\Omega$ with bounded, strictly positive densities $g_1,\dots,g_p$. Suppose that $\mu_0$ is an absolutely continuous Karcher mean for $\lambda$ and is supported on $\Omega$ with bounded strictly positive density $f$ there.  Then $\mu_0$ is the Wasserstein barycenter for coordinates $\lambda$ if one of the following holds:
    \begin{enumerate}[(1)]\setlength\itemsep{-4pt}
        \item $\Omega = \mathbb{R}^d$ and the densities $f,g_1,\dots,g_p$ are locally H\"older continuous.
        \item $\Omega$ is bounded and the densities $f,g_1,\dots,g_p$ are bounded below on $\Omega$.
    \end{enumerate}
\end{theorem}

Next we include a technical theorem which we will leverage in our proof of the convergence of the entries $A_{ij}$. Before doing so, we introduce the necessary background.
\begin{theorem}(\cite{brenier1991polar}) Let $\mu \in \mathcal{P}_{2,\emph{ac}}(\Omega)$ and $\nu \in \mathcal{P}(\Omega)$. Then
\begin{enumerate}
    \item There exists a solution $T_0$ to (\ref{eq:wass_dist}), with $T_0 = \nabla \varphi_0$, for a convex function $\varphi_0$ solving
    $$\inf_{\varphi \in L^1(\mu)} \int \varphi d\mu + \int \varphi^* d\nu$$
    where $\varphi^*$ is the convex conjugate of $\varphi$.
    \item If in addition $\nu \in \mathcal{P}_{2,\emph{ac}}(\Omega), $ then $\nabla \varphi_0^*$ is an admissible optimal transport map from $\nu$ to $\mu$.
\end{enumerate}
\end{theorem}
Let $\mathcal{C}^\alpha$ denote the space of functions possessing $\lfloor \alpha \rfloor$ continuous derivatives and whose $\lfloor \alpha \rfloor^{\text{th}}$ derivative is $(\alpha - \lfloor \alpha \rfloor$) H{\"o}lder smooth. Using these conventions, we require the technical conditions:

\textbf{A4:}  $\mu,\nu \in \mathcal{P}_{2,\text{ac}}(\Omega)$ for a compact convex set $\Omega$, with densities satisfying $p(x),q(x) \leq M$ and $q(x) \geq m > 0$ for all $x \in \Omega$.

\textbf{A5:}  $\varphi_0 \in \mathcal{C}^2(\Omega)$ and $\varphi^*_0 \in \mathcal{C}^{\alpha+1}(\Omega)$ for $\alpha > 1$.

\textbf{A6:}  There exist $l,L>0$ such that $T_0 = \nabla \varphi_0$, with $lI \preceq \nabla^2\varphi_0(x) \preceq LI$ for all $x \in \Omega$.

This first assumption is a specialization of \textbf{A1} and \textbf{A2}, restricting to the case where $\Omega$ is compact. The latter two can be thought of as regularity conditions on the optimal transport maps between base and reference measures. \textbf{A5} asserts that the optimal transport maps have smooth derivatives and \textbf{A6} controls the derivatives from both above and below. 

\subsection{Properties of The Entropic Map}
\label{SM:Entropic_Map}

Let $X_i \sim \mu_0, Y_i \sim \mu_1$ for $i=1,...,n$ be i.i.d. samples from their respective distributions. From these samples construct the empirical measures $\hat{\mu}_{0,n} = \frac{1}{n}\sum_{i=1}^n \delta_{X_i}$ and $\hat{\mu}_{1,n} = \frac{1}{n}\sum_{i=1}^n \delta_{Y_i}$.

For $\epsilon > 0$ and $n$ samples from each measures, the discrete entropically regularized OT problem can be written
$$\min_{\substack{\pi \in \mathbb{R}^{n \times n}_+ \\
        \pi\bm{1} = (1/n)\bm{1} \\
        \pi^T\bm{1} = (1/n)\bm{1}}} \sum_{j,k=1}^{n} \pi_{jk}||X_j - Y_k||_2^2 + \epsilon \pi_{jk}\log \pi_{jk}$$
This problem has a dual formulation \cite{genevay2019entropy} given by
$$
\max_{f,g \in \mathbb{R}^n} \frac{1}{n} \sum_{i=1}^n f_i+g_i - \epsilon\langle e^{f/\epsilon}, Ke^{g/\epsilon}\rangle
$$
where $K_{i,j} = \exp(-||X_i - Y_i||_2^2/\epsilon)$. $g_\epsilon$ is defined to be the optimal $g$ in the maximization above, parameterized by the setting of $\epsilon > 0.$  For the method proposed it is important that the $f_\epsilon$ and $g_\epsilon$ are chosen such that for all $x,y \in \mathbb{R}^d$ we have
\begin{align*}
    \frac{1}{n}\sum_{i=1}^n \exp\left(\frac{1}{\epsilon} \left [ f_\epsilon(x) + g_\epsilon(Y_i) - \frac{1}{2}||x-Y_i||_2^2 \right ]\right) &= 1 \\
    \frac{1}{n}\sum_{i=1}^n  \exp\left(\frac{1}{\epsilon} \left [ f_\epsilon(X_i) + g_\epsilon(y) - \frac{1}{2}||X_i-y||_2^2 \right ]\right) &= 1.
\end{align*}
This is always possible, and can be easily done in practice \cite{pooladian2021entropic}.

One of the primary motivations for selecting the entropic map as our map estimate is that $g_\epsilon$ can be efficiently computed \cite{peyre2020computational}, and it comes with the performance guarantee stated below.

In the following theorem the constants may depend on the dimension $d$, the diameter of the support $|\Omega|$, $M, m, l, L$, and $||\varphi_0^*||_{\mathcal{C}^{\alpha+1}}$.
\begin{theorem}(\cite{pooladian2021entropic} Theorem 3) \label{thm:PNW} Under \textbf{A4},\textbf{A5},\textbf{A6}, the entropic map $\hat{T} = T_{\epsilon,(n,n)}$ from $\hat{\mu}_n$ to $\hat{\nu}_n$ with regularization parameter $\epsilon \asymp n^{-1/(d' + \bar{\alpha} + 1)}$ satisfies
\begin{equation}
    \mathbb{E}||\hat{T} - T_0||^2_{L^2(\mu)} \lesssim (1 + I_0(\mu,\nu))n^{-\frac{(\bar{\alpha} + 1)}{2(d' + \bar{\alpha} + 1)}}\log n
\end{equation}
where $d' = 2\lceil d/2 \rceil$ and $\bar{\alpha} = \min(\alpha, 3)$ and $I_0(\mu,\nu)$ is the integrated Fisher information between $\mu$ and $\nu$.
\end{theorem}
For more information on $I_0(\mu, \nu)$ see \cite{chizat2020faster}. We remark that when $\alpha \geq 2$ then under \textbf{A4},\textbf{A5},\textbf{A6}, it has been shown $I_0(\mu,\nu) \leq C$ for a positive constant $C$ \cite{chizat2020faster}. The notation $a \asymp b$ means that there exists positive constants $c,C$ such that $ca \leq b \leq Ca$ and $a \lesssim b$ means that there is a positive constant $C$ such that $a \leq Cb$. Note that this result still holds without the convexity condition in \textbf{A4}. We include it for ease of exposition below and to ensure that the interior of an $\Omega$ that satisfies \textbf{A4} is convex and therefore satisfies \textbf{A3}.

\section{Proofs of Theoretical Results}

\subsection{Proof of Proposition \ref{prop:Prop1}}\label{SM:Prop1Proof}

\begin{proof}
    We first apply the definitions given in Section \ref{sec:Theory} and leverage (\ref{eq:grad}):
    \begin{align*}
        ||\nabla G_{\lambda}(\mu_0)||_{\mu_0}^2 &= ||-\sum_{i=1}^p \lambda_i (T_{\mu_0\rightarrow\mu_i} - \id)||_{\mu_0}^2 \\
        = &\int_{\mathbb{R}^d} \left \langle \sum_{i=1}^p \lambda_i (T_{\mu_0\rightarrow\mu_i} - \id), \sum_{j=1}^p \lambda_j (T_{\mu_0\rightarrow\mu_j} - \id)  \right \rangle d\mu_0 \\
        = &\int_{\mathbb{R}^d} \sum_{i,j=1}^p \lambda_i \lambda_j \langle T_{\mu_0\rightarrow\mu_i} - \id, T_{\mu_0\rightarrow\mu_j} - \id\rangle d\mu_0\\
        = &\sum_{i,j=1}^{p} \lambda_i\lambda_j A_{ij} = \lambda^TA\lambda.
    \end{align*}
    To establish the convexity of the function in terms of $\lambda$, it is sufficient to show that $A$ is symmetric positive semi-definite. Observe that $A$ is a Gram matrix of the set $\{T_{\mu_0\rightarrow\mu_i} - \id\}_{i=1}^p \subset T_{\mu_0}\Pac$ with the inner product in $T_{\mu_0}\Pac$. It is classical that all Gram matrices are symmetric positive semi-definite \citep{schwerdtfeger1961introduction} and this completes the proof.
\end{proof}

\subsection{Proof of Theorem \ref{thm:main}}\label{SM:Thm1Proof}

\begin{proof}
    The convexity of the objective is established in Proposition \ref{prop:Prop1}, and the feasible set $\Delta^p$ is trivially convex. 
    
    By Theorem \ref{thm:karcher}, when assumptions \textbf{A1}-\textbf{A3} hold, all Karcher means are barycenters, and therefore $\mu_0 = \nu_{\lambda}$, and it is always the case that barycenters are Karcher means. Therefore it is sufficient to check that the minimum value is 0 if and only if $\mu_0$ is a Karcher mean.
    
    In the case that the minimum is zero, there is some $\lambda_*$ such that $(\lambda_*)^TA\lambda_* = 0$ and we can apply Proposition \ref{prop:Prop1} to show that $\mu_0$ is a Karcher mean for $\lambda_*$. 
    Similarly if $\mu_0$ is a Karcher mean for some $\lambda_*$ by definition we have $||\nabla G_{\lambda_*}(\mu_0)||_{\mu_0}^2 = 0$ and again by Proposition \ref{prop:Prop1} we have that $\lambda_*$ achieves the minimum of zero in the optimization.
\end{proof}

\subsection{Proof of Corollary \ref{cor:gaussian}}\label{SM:Cor1Proof}

\begin{proof}
    \textbf{A1},\textbf{A2},\textbf{A3} are clearly satisfied in the Gaussian case with $\Omega = \mathbb{R}^d$ since the pdf of a non-degenerate multivariate Gaussian is positive everywhere and Lipschitz continuous. Furthermore the barycenter of zero-mean Gaussians is itself a zero-mean Gaussian (\citep{agueh2011barycenters} Theorem 6.1).
    
    Therefore we can apply Theorem \ref{thm:main} and all that remains is to calculate $A_{ij}$ as follows. Noting that the $C_i$ are all symmetric \cite{agueh2011barycenters} we have
    \begin{align*}
    A_{ij} &= \int_{\mathbb{R}^d} \left \langle 
            (T_i(x) - \id(x)),
            (T_j(x) - \id(x)) 
        \right \rangle dS_0(x)  \\
    &= \int_{\mathbb{R}^d} \left \langle 
            (C_i - I)x, 
            (C_j - I)x 
        \right \rangle dS_0(x) & (\text{Eq. \ref{eq:map_mat}})  \\
    &= \underset{X \sim \mathcal{N}(0,S_0)}{\mathbb{E}}\left [ 
            X^T \left (C_i - I \right )^T 
            \left (C_j - I \right ) X 
        \right ]  \\
    &= \underset{X \sim \mathcal{N}(0,S_0)}{\mathbb{E}}\left [ 
            X^T \left (C_i - I \right ) 
            \left (C_j - I \right ) X 
        \right ]  & (C_i \text{ all sym.})  \\
    &= \Tr \left ( 
            \left (C_i - I \right ) 
            \left (C_j - I \right ) S_0 
        \right ).
\end{align*}
The last equality comes from the identity
$\underset{X \sim \mathcal{N}(0,S)}{\mathbb{E}}
    \left [ X^T B X \right ]  = \Tr(BS)$
which holds for all $B\in \mathbb{R}^{d\times d}$ and $S \in \mathbb{S}^d_{++}.$
\end{proof}

\subsection{Proof of Theorem \ref{thm:projection}} \label{sec:projection}

\begin{proof}
    In this setting, the barycenter $\nu_\lambda$ can be expressed as
    $$
    \nu_\lambda = \left [ \sum_{i=1}^p \lambda_i T_0^i \right ] \# \mu_0
    $$
    (\cite{panaretos2020invitation} Theorem 3.1.9).  Together with Proposition \ref{prop:Prop1}, this implies (recalling $T_i = T_0^i$) 
    $$
    W_2^2(\mu_0, \nu_\lambda) = \int || x- \sum_{i=1}^p \lambda_i T_i(x) ||_2^2 d\mu_0(x) = ||\nabla G_{\lambda}(\mu_0)||_{\mu_0}^2 = \lambda^TA\lambda
    $$
    and therefore
    $$
    \argmin_{\lambda \in \Delta^p} W_2^2(\mu_0, \nu_{\lambda}) = \argmin_{\lambda \in \Delta^p} \lambda^TA\lambda.
    $$
\end{proof}

\subsection{Exact Statement and Proof of Theorem \ref{thm:convergence}}\label{SM:Thm2Proof}

\begin{theorem}(Formal Version of Theorem \ref{thm:convergence}, Convergence Rate of $A_{ij}$)\label{thm:Thm3Proof_Exact}
    Let $T_1,T_2$ be the optimal transport maps from $\mu_0$ to $\mu_1$, $\mu_2$ respectively. Suppose that \textbf{A4, A5, A6} are satisfied for both pairs $(\mu_0, \mu_1)$ and $(\mu_0, \mu_2)$. Let $X_1,...,X_{2n} \sim \mu_0, Y_1,...,Y_n \sim \mu_1, Z_1,...,Z_n \sim \mu_2$. Let $\hat{T}_1$ and $\hat{T}_2$ be the entropic maps from $\mu_0$ to $\mu_1$ and $\mu_2$ respectively, both with $\epsilon \asymp n^{-1/(d'+\Bar{\alpha}+1)}$ and computed using $X_{1},...,X_{n}$. Then we have
    \begin{align*}
       &\mathbb{E}\left [ \left | \int \langle T_1 - \id, T_2 - \id \rangle d\mu_0 - \frac{1}{n} \sum_{i=n+1}^{2n} \langle \hat{T}_1(X_i) - X_i, \hat{T}_2(X_i) - X_i \rangle \right | \right ] \\
            \lesssim& \frac{1}{\sqrt{n}} + n^{-\frac{\Bar{\alpha}+ 1}{4(d' + \Bar{\alpha} + 1) }}\sqrt{\log n}\sqrt{1+I_0(\mu_0,\mu_1) + I_0(\mu_0, \mu_2)}
    \end{align*}
    where  $d' = 2\lceil d/2 \rceil$, $\Bar{\alpha} = \alpha \wedge 3$ and $I_0(\mu_0,\mu_1)$ is the integrated Fisher information along the Wasserstein geodesic between $\mu$ and $\mu_1$.
\end{theorem}

Before we begin the proof, we remark that by writing $\mu_i,\mu_j$ instead of $\mu_1,\mu_2$, the random variable in the expectation is just $|A_{ij} - \hat{A}_{ij}|$ and therefore this result allows us to control the entry-wise deviations of $\hat{A}_{ij}$ from their true values.

\begin{proof}
    Throughout, unless otherwise noted, the expectation is with respect to all of $X_1,...,X_{2n}, Y_1,...,Y_n$ and $Z_1,...,Z_n$. We also use the notations $X^n$ to denote the set of variables $(X_1,...,X_n)$, and similarly for $Y^n$ and $Z^n$.  
    
    First observe that we can assume without loss of generality that $0 \in \Omega$; see \cite{peyre2020computational} Remark 2.9 for invariance of $T_1,T_2$ under translation. For the translation invariance of $\hat{T}_1,\hat{T}_2$, observe that both obtaining $g_\epsilon$, and evaluating $\hat{T}_1, \hat{T}_2$ at fixed points, requires only the distances $||X_i - Y_j||_2^2 = ||(X_i - t) - (Y_j - t)||_2^2$.  We calculate:
    
    \begin{align*}
        &\mathbb{E}\left [ \left | \int \langle T_1 - \id, T_2 - \id \rangle d\mu_0 - \frac{1}{n} \sum_{i=n+1}^{2n} \langle \hat{T}_1(X_i) - X_i, \hat{T}_2(X_i) - X_i \rangle \right | \right ]  \\
        =& \mathbb{E}\left[ \left | 
            \int \langle T_1 - \id, T_2 - \id \rangle d\mu_0 - \frac{1}{n}\sum_{i=n+1}^{2n} \langle T_1(X_i) - X_i, T_2(X_i) - X_i\rangle \right. \right .\\
            & \hspace{1cm} + \left.\left.\frac{1}{n}\sum_{i=n+1}^{2n} \langle T_1(X_i) - X_i, T_2(X_i) - X_i\rangle 
            - \frac{1}{n} \sum_{i=n+1}^{2n} \langle \hat{T}_1(X_i) - X_i, \hat{T}_2(X_i) - X_i \rangle \right | \right ] \\
        \leq& \mathbb{E}\left[ \left | 
            \int \langle T_1 - \id, T_2 - \id \rangle d\mu_0 - \frac{1}{n}\sum_{i=n+1}^{2n} \langle T_1(X_i) - X_i, T_2(X_i) - X_i\rangle \right| \right ]\\
            & \hspace{1cm} + \mathbb{E}\left [\left|\frac{1}{n}\sum_{i=n+1}^{2n} \langle T_1(X_i) - X_i, T_2(X_i) - X_i\rangle 
            - \frac{1}{n} \sum_{i=n+1}^{2n} \langle \hat{T}_1(X_i) - X_i, \hat{T}_2(X_i) - X_i \rangle \right | \right ] \\
    \end{align*}
    We consider the first and second terms separately. For brevity, let $h$ be the function defined by
    $$h(x) = \langle T_1(x) - x, T_2(x) - x\rangle.$$
    Note that $h$ is bounded on $\Omega$ since 
    $$|h(x)| = |\langle T_1(x) - x, T_2(x) - x\rangle| \leq ||T_1(x) - x||_2||T_2(x) - x||_2 \leq (2|\Omega|)(2|\Omega|) = 4|\Omega|^2,$$
    which implies by Popoviciu's inequality \cite{popoviciu1935equations} that
    $\V[h(X)] \leq 16|\Omega|^4$.
    From this it follows
    \begin{align*}
        &\mathbb{E}\left[ \left | 
            \int \langle T_1 - \id, T_2 - \id \rangle d\mu_0 - \frac{1}{n}\sum_{i=n+1}^{2n} \langle T_1(X_i) - X_i, T_2(X_i) - X_i\rangle \right| \right ] \\
            =& \mathbb{E} \left[\left|\mathbb{E}_{X \sim \mu_0}[h(X)] - \frac{1}{n}\sum_{i=n+1}^{2n} h(X_i)\right|\right] \\
            \leq& \sqrt{\frac{\V[h(X)]}{n}} \leq \sqrt{\frac{16|\Omega|^4}{n}}  \lesssim \frac{1}{\sqrt{n}}
    \end{align*}
    where we have used that for all i.i.d. random variables $U,U_1,...,U_n$ with finite variance 
    $$\mathbb{E}[|\mathbb{E}[U] - \frac{1}{n}\sum_{i=1}^n U_i|] \leq \sqrt{\frac{\V[U]}{n}}.$$
    We now handle the second term:
    \begin{align*}
        &\mathbb{E}\left [\left|\frac{1}{n}\sum_{i=n+1}^{2n} \langle T_1(X_i) - X_i, T_2(X_i) - X_i\rangle 
            - \frac{1}{n} \sum_{i=n+1}^{2n} \langle \hat{T}_1(X_i) - X_i, \hat{T}_2(X_i) - X_i \rangle \right | \right ]  \\
        \leq& \frac{1}{n} \sum_{i=n+1}^{2n} \mathbb{E}\left [ \left | \langle T_1(X_i) - X_i, T_2(X_i) - X_i \rangle 
            - \langle \hat{T}_1(X_i) - X_i, \hat{T}_2(X_i) - X_i \rangle \right |\right ]  \\
        =& \mathbb{E}\left [ \left | \langle T_1(X_{n+1}) - X_{n+1}, T_2(X_{n+1}) - X_{n+1} \rangle 
            - \langle \hat{T}_1(X_{n+1}) - X_{n+1}, \hat{T}_2(X_{n+1}) - X_{n+1} \rangle \right |\right ]  \\
        =& \mathbb{E}\left [\left | \langle T_1(X_{n+1}), T_2(X_{n+1})\rangle - \langle \hat{T}_1(X_{n+1}), \hat{T}_2(X_{n+1})\rangle \right. \right. \\
        &\hspace{3cm} + \left. \left. \langle \hat{T}_1(X_{n+1}) - T_1(X_{n+1}), X_{n+1} \rangle + \langle \hat{T}_2(X_{n+1}) - T_2(X_{n+1}), X_{n+1} \rangle \right| \right ]  \\
        \leq& \mathbb{E}\left [ \left | \langle T_1(X_{n+1}), T_2(X_{n+1})\rangle - \langle \hat{T}_1(X_{n+1}), \hat{T}_2(X_{n+1})\rangle \right | \right ] \\
        & \hspace{3cm} + \mathbb{E}\left [ \left | \langle \hat{T}_1(X_{n+1}) - T_1(X_{n+1}), X_{n+1} \rangle \right | \right ] 
        + \mathbb{E}\left [ \left | \langle \hat{T}_2(X_{n+1}) - T_2(X_{n+1}), X_{n+1} \rangle \right | \right ]
    \end{align*}
    
    Here we can control the three terms separately. We first take care of the middle term and the last term follows by an identical argument.
    \begin{align*}
        &\mathbb{E} \left [ \left | \langle \hat{T}_1(X_{n+1}) - T_1(X_{n+1}), X_{n+1} \rangle \right | \right ] \\
        =& \mathbb{E}_{X^n,Y^n, X_{n+1}} \left [ \left | \langle \hat{T}_1(X_{n+1}) - T_1(X_{n+1}), X_{n+1} \rangle \right | \right ] \\
        \leq& \mathbb{E}_{X^n,Y^n, X_{n+1}} \left [ ||\hat{T}_1(X_{n+1}) - T_1(X_{n+1})||_2||X_{n+1}||_2\right ] & (\text{Cauchy-Schwarz}) \\
        \leq& |\Omega|\mathbb{E}_{X^n,Y^n, X_{n+1}} \left [ ||\hat{T}_1(X_{n+1}) - T_1(X_{n+1})||_2\right ] & (||X_{n+1}|| \leq |\Omega|) \stepcounter{equation}\tag{\theequation}\label{eq:term_copy} \\
        =& |\Omega|\mathbb{E}_{X^n,Y^n, X_{n+1}}\left [ \mathbb{E}_{X_{n+1}} ||\hat{T}_1(X_{n+1}) - T_1(X_{n+1})||_2 \right ] \\
        =& |\Omega|\mathbb{E}_{X^n,Y^n, X_{n+1}}\left [ \mathbb{E}_{X_{n+1}} \sqrt{||\hat{T}_1(X_{n+1}) - T_1(X_{n+1})||_2^2} \right ] \\ 
        \leq& |\Omega|\mathbb{E}_{X^n,Y^n, X_{n+1}}\left [ \sqrt{\mathbb{E}_{X_{n+1}} ||\hat{T}_1(X_{n+1}) - T_1(X_{n+1})||_2^2} \right ] & (\text{Jensen})\\
        =&|\Omega|\mathbb{E}_{X^n,Y^n} \left [ ||\hat{T}_1 - T_1||_{L^2(\mu_0)} \right ] \\
        \leq& |\Omega|\sqrt{\mathbb{E}_{X^n,Y^n}\left [ ||\hat{T}_1 - T_1||_{L^2(\mu_0)}^2 \right ]} \\
        \lesssim& \sqrt{(1+I_0(\mu_0,\mu_1))n^{-\frac{\Bar{\alpha}+ 1}{2(d' + \Bar{\alpha} + 1)}}\log(n)}. & (\text{Theorem \ref{thm:PNW}})
    \end{align*}
    As mentioned above the third term is exactly the same except replacing $T_1$ and $\hat{T}_1$ with $T_2$ and $\hat{T_2}$ as well as $Y^n$ with $Z^n$.

    To control the first term we start with a different trick and then end up following the same pattern as in the bound above.
    \begin{align*}
        &\mathbb{E}\left [ \left | \langle T_1(X_{n+1}), T_2(X_{n+1})\rangle - \langle \hat{T}_1(X_{n+1}), \hat{T}_2(X_{n+1})\rangle \right | \right ] \\
        =& \mathbb{E}\left [ \left | \langle T_1(X_{n+1}) - \hat{T}_1(X_{n+1}), T_2(X_{n+1})\rangle + \langle \hat{T}_1(X_{n+1}), T_2(X_{n+1}) - \hat{T}_2(X_{n+1})\rangle \right | \right ] \\
        \leq& \mathbb{E}\left[ \left | \langle T_1(X_{n+1}) - \hat{T}_1(X_{n+1}), T_2(X_{n+1})\rangle \right | \right] + \mathbb{E}\left[ \left | \langle \hat{T}_1(X_{n+1}), T_2(X_{n+1}) - \hat{T}_2(X_{n+1})\rangle \right | \right] \\
        \leq&  \mathbb{E}\left[ ||T_1(X_{n+1}) - \hat{T}_1(X_{n+1})||_2 ||T_2(X_{n+1})|| \right] + \mathbb{E}\left[ || \hat{T}_1(X_{n+1}) ||_2 ||T_2(X_{n+1}) - \hat{T}_2(X_{n+1})||_2 \right] \\
        \leq& |\Omega| \mathbb{E}\left[ ||T_1(X_{n+1}) - \hat{T}_1(X_{n+1})||_2 \right] + |\Omega|\mathbb{E}\left[ ||T_2(X_{n+1}) - \hat{T}_2(X_{n+1})||_2 \right] \\
        =& |\Omega| \mathbb{E}_{X^n,Y^n,X_{n+1}}\left[ ||T_1(X_{n+1}) - \hat{T}_1(X_{n+1})||_2 \right] + |\Omega|\mathbb{E}_{X^n,Z^n,X_{n+1}}\left[ ||T_2(X_{n+1}) - \hat{T}_2(X_{n+1})||_2 \right]. \stepcounter{equation}\tag{\theequation}\label{eq:two_terms}
    \end{align*}
    Above we have made use of the fact that both $T_2(X_1)$ and $\hat{T_1}(X_1)$ are contained in $\Omega$ which implies that $||T_2(X_1)||_2 \leq |\Omega|$.  $T_2(X_1) \in \Omega$ follows from the fact that $T_2$ is a map from $\mu_0$ to $\mu_2$. To see that $||\hat{T}_1(X_1)|| \leq |\Omega|$, note that from equation (\ref{eq:entropic_map}) $\hat{T}_1(X_1)$ is a convex combination of $Y_1,...,Y_n \in \Omega$ and each of which satisfies $||Y_i||_2 \leq |\Omega|$. 
    
    Observe that both terms in (\ref{eq:two_terms}) have already been controlled, starting from Equation (\ref{eq:term_copy}) (the latter requires replacing $T_1,\hat{T}_1$, and $Y^n$ with $T_2,\hat{T}_2$, and $Z^n$ respectively). Applying the bound derived there and combining terms terms we have the following result
    \begin{align*}
        &\mathbb{E}\left [ \left | \int \langle T_1 - \id, T_2 - \id \rangle d\mu_0 - \frac{1}{n} \sum_{i=n+1}^{2n} \langle \hat{T}_1(X_i) - X_i, \hat{T}_2(X_i) - X_i \rangle \right | \right ] \\
        & \lesssim \frac{1}{\sqrt{n}} + 2\sqrt{(1+I_0(\mu_0,\mu_1))n^{-\frac{\Bar{\alpha}+ 1}{2(d' + \Bar{\alpha} + 1)}}\log(n)} + 2\sqrt{(1+I_0(\mu_0,\mu_2))n^{-\frac{\Bar{\alpha}+ 1}{2(d' + \Bar{\alpha} + 1)}}\log(n)} \\
        &= \frac{1}{\sqrt{n}} + n^{-\frac{\Bar{\alpha}+ 1}{4(d' + \Bar{\alpha} + 1)}}\sqrt{\log n}\left (\sqrt{1+I_0(\mu_0,\mu_1)} + \sqrt{1+I_0(\mu_0,\mu_2)} \right ) \\
        &\lesssim \frac{1}{\sqrt{n}} + n^{-\frac{\Bar{\alpha}+ 1}{4(d' + \Bar{\alpha} + 1)}}\sqrt{\log n}\sqrt{1+I_0(\mu_0,\mu_1)+I_0(\mu_0,\mu_2)}.
    \end{align*}
\end{proof}

\subsection{Proof of Corollary \ref{cor:consistency}}\label{SM:Cor2Proof}

\begin{proof}
Let $B_{n}$ denote the entrywise bound in Theorem \ref{thm:convergence}.  Noting that $\hat{\lambda}^{T}\hat{A}\hatlambda\le \lambda_{*}^{T}\hat{A}\lambda_{*}$ by construction, we estimate 
\begin{align*}
    \E[\hat{\lambda}^{T}A\hatlambda] 
    &= \E\left[\hat{\lambda}^{T}(A-\hat{A})\hatlambda\right]+\E[\hat{\lambda}^{T}\hat{A}\hat{\lambda}] \\
    &\leq \E\left[|\hat{\lambda}^{T}(A-\hat{A})\hatlambda|\right] +\E[\lambda_{*}^{T}\hat{A}\lambda_{*}]\\
    &= \E\left[|\hat{\lambda}^{T}(A-\hat{A})\hatlambda|\right]+\E[\lambda_{*}^{T}(\hat{A}-A)\lambda_{*}]\,\, & (\lambda_*^TA\lambda_* = 0) \\
    &= \E\left[\bigg| \sum_{i,j=1}^p (\hat{\lambda})_{i}(\hat{\lambda})_{j} (A - \hat{A})_{ij}\bigg|\right]+
        \E\left[\bigg| \sum_{i,j=1}^p (\lambda_{*})_{i}(\lambda_{*})_{j} (A - \hat{A})_{ij}\bigg|\right] \\
    &\le \E\left[ \sum_{i,j=1}^p (\hat{\lambda})_{i}(\hat{\lambda})_{j} |A_{ij} - \hat{A}_{ij}|\right] + 
        \E\left[\sum_{i,j=1}^p (\lambda_{*})_{i}(\lambda_{*})_{j} |A_{ij} - \hat{A}_{ij}|\right] & (\hat{\lambda}, \lambda_* \in \Delta^p, \text{ triangle ineq.})\\
    &\le 2\E\left[\sum_{i,j=1}^p |A_{ij} - \hat{A}_{ij}|\right] \\
    &= 2\sum_{i,j=1}^p \E[|A_{ij} - \hat{A}_{ij}|] \\
    & \lesssim 2p^{2}B_{n} & (\text{Theorem \ref{thm:convergence}})
\end{align*}

Since $A$ is positive semidefinite and by assumption $\lambda_{*}\in\Delta^{p}$ satisfies $\lambda_{*}^{T}A\lambda_{*}=0$, it follows that $\lambda_{*}$ is an eigenvector of $A$ with eigenvalue 0.  Let $0<\alpha_{2}\le\dots\le\alpha_{p}$ be the non-zero eigenvalues of $A$ with associated orthonormal eigenvectors $v_{2},\dots,v_{p}$. Orthogonally decompose $\hatlambda=\hat{\beta}\lambda_{*}+\hatlambda_{\perp}$, where $\hat{\beta}\in \mathbb{R}$ and $\hatlambda_{\perp}$ is in the span of $\{v_{2},\dots,v_{p}\}$.  Note that $\hat{\beta}$ and $\hat{\lambda}_{\perp}$ are random.  Then,

\begin{align*}\E[\|\hatlambda-\hat{\beta}\lambda_{*}\|_{2}^2]=&\E[\|\hat{\lambda}_{\perp}\|_{2}^2]\\
=&\E\left[\sum_{i=2}^{p}|v_{i}^{T}\hat{\lambda}_{\perp}|^{2}\right]\\
\le& \frac{1}{\alpha_{2}}\E\left[\sum_{i=2}^{p}\alpha_{i}|v_{i}^{T}\hat{\lambda}_{\perp}|^{2}\right]\\
=&\frac{1}{\alpha_{2}}\E[|(\hat{\lambda}_{\perp})^{T}A\hat{\lambda}_{\perp}|]\\
=&\frac{1}{\alpha_{2}}\E[|\hat{\lambda}^{T}A\hatlambda|]\\
\lesssim & \frac{2 p^{2}}{\alpha_{2}}B_{n}.
\end{align*}

Summing both sides of the equation $\hat{\lambda} = \hat{\beta} \lambda_*  + \hat{\lambda}_{\perp}$ and recalling $\lambda_{*}, \hat{\lambda}\in\Delta^{p}$ yields 
\begin{align*}1 &= \hat{\beta} +\sum_{j=1}^{p}(\hat{\lambda}_\perp)_{j}\\
&\le\hat{\beta}+\|\hat{\lambda}_\perp\|_{1}\\
&\le \hat{\beta}+\sqrt{p}\|\hat{\lambda}_\perp\|_{2},
\end{align*}
which implies that $\mathbb{E}[(1 - \hat{\beta})^2] \leq p \mathbb{E}[\| \hat{\lambda}_\perp\|_2^2] \lesssim \frac{2p^3}{\alpha_2} B_n$.

Finally, we use the fact that $\hatlambda-\hat{\beta}\lambda_{*}=\hat{\lambda}_{\perp}$ and $(\hat{\beta}-1)\lambda_{*}$ are orthogonal to bound: 

\begin{align*}
    \E[\|\hatlambda-\lambda_{*}\|_{2}^{2}] 
    &= \E[\|\hatlambda-\hat{\beta}\lambda_{*}\|_{2}^{2}+\|(\hat{\beta}-1)\lambda_{*}\|_{2}^{2}] \\
    &= \E[\|\hatlambda-\hat{\beta}\lambda_{*}\|_{2}^{2}] + \E[\|(\hat{\beta}-1)\lambda_{*}\|_{2}^{2}]\\
    &\lesssim \frac{2 p^{2}}{\alpha_{2}}B_{n}+\E[(\hat{\beta}-1)^{2}\|\lambda_{*}\|_{2}^{2}]\\
    &\le \frac{2 p^{2}}{\alpha_{2}}B_{n} + \E[(\hat{\beta}-1)^{2}]\\
    &\lesssim \frac{2 p^{2}}{\alpha_{2}}B_{n} + \frac{2p^3}{\alpha_2} B_n
\end{align*}as desired. 

\end{proof}

\subsection{Karcher Means are Barycenters under \textbf{A4}-\textbf{A6}}

Corollary \ref{cor:consistency} (which holds under \textbf{A4}-\textbf{A6}) can be combined with Theorem \ref{thm:main} (which holds under \textbf{A1}-\textbf{A3}) as follows.  Suppose $||G_{\lambda}(\nu)||_{\nu}^2 = 0$ for some $\lambda \in \Delta^p$ and reference measures $\{\mu_i\}_{i=1}^{p}$, and assume that the collection $\{\nu\} \cup \{\mu_i\}_{i=1}^p$ satisfy assumptions \textbf{A4}-\textbf{A6}. We aim to show that this implies that $\nu$ is a barycenter for the measures $\{\mu_{i}\}_{i=1}^{p}$.

Let $\tilde{\mu}_i$ be the measure with density $$d\tilde{\mu}_i(x) = \begin{cases}
    d\mu_i(x) & x \in \Omega^\circ, \\
    0 & \text{otherwise}
\end{cases}$$
where $\Omega^{\circ}$ is the interior of $\Omega$.  Define $\tilde{\nu}$ similarly.

We first remark that 
$$1 = \mu_i[\Omega] = \mu_i[\Omega^\circ] + \mu_i[\partial\Omega] = \Tilde{\mu}_i[\Omega^\circ] + 0 = \Tilde{\mu}_i[\Omega^\circ]$$
and therefore $\Tilde{\mu}_i$ (and analogously $\nu$) are valid measures (note we have used that $\mu_{i}[\partial \Omega]=0$, which follows from $\mu_{i}$ being a.c. and $\partial \Omega$ having Lebesgue measure 0).  It follows from \textbf{A4} that the collection $\{\nu\} \cup \{\mu_i\}_{i=1}^p$ satisfies \textbf{A2} and \textbf{A3}. Furthermore, the set $\Omega^\circ$ is open and convex and therefore \textbf{A1} is also satisfied. Therefore if $\tilde{\nu}$ is a Karcher mean for $\{\tilde{\mu}_i\}$ then it is a barycenter for them as well.

We next show that $\tilde{\nu}$ is a Karcher mean of $\{\tilde{\mu}_i\}_{i=1}^p$ with coordinate $\lambda$. Consider the maps $\tilde{T}_i:\Omega^\circ \rightarrow \Omega^\circ$ which are defined by $\tilde{T}_i(x) = T_i(x)$ and are undefined outside $\Omega^\circ$
where $T_i$ is the optimal transport map from $\nu \rightarrow \mu_i$. We first require that $\tilde{T}_i$ is in fact optimal for $\tilde{\nu}$ and $\tilde{\mu}_i$ and that it is well-defined $\tilde{\nu}$-a.e. 

To show the latter, note that $T_i$ is a map from $\nu$ to $\mu_i$ and therefore $\nu[T_i^{-1}(\partial \Omega)] = \mu_i[\partial \Omega] = 0$ which implies that $T_i(x) \in \Omega^\circ$ for $\nu$-a.e. $x$. Furthermore $\tilde{\nu}[\Omega^\circ] = \nu[\Omega^\circ] = 1$ and therefore the map $\Tilde{T}_i$ is well-defined $\tilde{\nu}$-a.e.

To show that the map $\tilde{T}_i$ is optimal, one only needs to remark that for any measurable $B \subset \Omega$ that $T_i$ transports mass optimally between $B$ and $T_i(B)$ (since otherwise $T_i$ would fail to be optimal). Applying this fact to $\Omega^\circ$ and using the considerations above, shows that $\tilde{T}_i$ is optimal for transporting between $\tilde{\nu}$ and $\tilde{\mu}_i$.

Having established that the $\tilde{T}^i$ are optimal, we can now calculate $||G_\lambda(\tilde{\nu})||_{\tilde{\nu}}^2$:
\begin{align*}
    ||G_\lambda(\tilde{\nu})||_{\tilde{\nu}}^2 &= \sum_{i,j=1}^p \lambda_i\lambda_j \int_{\Omega^\circ} \left \langle \tilde{T}_i(x) - x, \tilde{T}_j(x) - x \right \rangle d\tilde{\nu}(x) \\
    &= \sum_{i,j=1}^p \lambda_i\lambda_j\int_{\Omega^\circ} \left \langle T_i(x) - x, T_j(x) - x \right \rangle d\nu(x) \\
    &= \sum_{i,j=1}^p \lambda_i\lambda_j\int_{\Omega^\circ} \left \langle T_i(x) - x, T_j(x) - x \right \rangle d\nu(x) + \lambda_i\lambda_j \int_{\partial \Omega} \left \langle T_i(x) - x, T_j(x) - x \right \rangle d\nu(x) \\
    &= \sum_{i,j=1}^p \lambda_i\lambda_j \int_{\Omega} \left \langle T_i(x) - x, T_j(x) - x \right \rangle d\nu(x) \\
    &= ||G_\lambda(\nu)||_{\nu}^2 = 0
\end{align*}
which shows that $\tilde{\nu}$ is a Karcher mean for $\{\tilde{\mu}_i\}$, and is also a barycenter of these measures. 
To conclude, we only require the fact that if two measures $\xi,\tilde{\xi} \in \Pac$ differ only on a set of measure zero, then for any $\gamma \in \Pac$ we have $W_2(\xi, \gamma) = W_2(\tilde{\xi}, \gamma)$. Applying this fact repeatedly we have 
\begin{align*}
    \min_{\xi} \frac{1}{2}\sum_{i=1}^n  \lambda_i W_2^2(\xi, \mu_i) &= \min_{\xi} \frac{1}{2}\sum_{i=1}^n  \lambda_i W_2^2(\xi, \tilde{\mu}_i) \\
    &= \frac{1}{2}\sum_{i=1}^n \lambda_i W_2^2(\tilde{\nu}, \tilde{\mu}_i) \\
    &= \frac{1}{2}\sum_{i=1}^n \lambda_i W_2^2(\nu, \tilde{\mu}_i) \\
    &= \frac{1}{2}\sum_{i=1}^n \lambda_i W_2^2(\nu, \mu_i) 
\end{align*}
which shows that $\nu$ is indeed a barycenter of $\mu_i$, as desired.

\section{Covariance Estimation Experiment Details} \label{SM:Covariance}

\subsection{Methods for Maximum Likelihood Estimation}

In order to solve the maximum likelihood estimation problem we differentiate through a truncated version of \citep{chewi2020gradient} Algorithm 1 and perform projected gradient descent. This is implemented using the auto-differentiation library PyTorch. The procedure is summarized in Algorithm \ref{alg:MLE}.
\begin{algorithm} 
\caption{MLE}\label{alg:MLE}
\begin{algorithmic}
\STATE {\bfseries Input:} $\{S_i\}_{i=0}^p$, $\eta > 0$, MaxIters > 0, FPIters > 0, SQIters > 0 
\STATE $i \leftarrow 0, \lambda  \leftarrow (1/p)\bm{1}$
\WHILE{Not Converged and $i < $ MaxIters}
    \STATE $i \leftarrow i + 1$ 
    \STATE $\nabla \mathcal{L}(\lambda) \leftarrow $ BackPropLoss($\{S_j\}_{j=0}^p$, $\lambda$, FPIters, SQIters)
    \STATE $\lambda \leftarrow $ SimplexProject($\lambda - \eta \nabla \mathcal{L}(\lambda)$)
\ENDWHILE
\STATE \textbf{Return} $\lambda$
\end{algorithmic}
\end{algorithm}
Here SimplexProject$(x)$ is defined as 
$$\text{SimplexProject}(x) = \argmin_{y \in \Delta^p} ||x - y||_2.$$
which enforces the constraints on $\lambda$ at each iteration.

To compute $\nabla \mathcal{L}(\lambda)$ we use auto-differentiation to obtain a gradient of the procedure given in Algorithm \ref{alg:cl}. In this procedure the square root of a matrix is computed using SQIters number of Newton-Schulz iterations. The forward pass of the loss computation is given in Algorithm \ref{alg:cl}, and the gradient is obtained by back propagation through it.
\begin{algorithm} 
\caption{ComputeLoss} \label{alg:cl}
\begin{algorithmic}
\STATE {\bfseries Input:} $\{S_i\}_{i=0}^p$, $\lambda$, FPIters > 0, SQIters > 0
\STATE BC $\leftarrow S_0$ 
\FOR{$j = 1,..., $ FPIters}
    \STATE BC\_root = SquareRoot(BC, SQIters) 
    \STATE  BC\_root\_inv = Invert(BC\_root) 
    \STATE BC = BC\_root\_inv $ \left ( \sum_{i=1}^p \lambda_i \text{SquareRoot( BC\_root, SQIters,} S_i \text{  BC\_root } ) \right )$ BC\_root\_inv 
\ENDFOR
\STATE \textbf{Return} Tr($\text{BC}^{-1}S_0$) + log det BC 
\end{algorithmic}
\end{algorithm}

The parameters that we chose in our experiments are given in Table \ref{tab:ag_params}.
\begin{table}[]
\centering
\begin{tabular}{|c|c|}
\hline
Parameter & Value  \\ \hline
$\eta$    & 0.0003 \\ \hline
MaxIters  & 500    \\ \hline
FPIters   & 10     \\ \hline
SQIters   & 10     \\ \hline
\end{tabular}
\caption{\label{tab:ag_params}Parameters used when performing the maximum likelihood estimation.}
\end{table}
These parameters were sufficient to ensure that both the matrix square roots and fixed point iterations converged, and $\lambda$ always converged before the final iteration was reached.

\subsection{Considerations for Instability of MLE}

Due to the requirement of differentiating through a matrix inverse in Algorithm \ref{alg:MLE} it is possible that the procedure may fail to numerically converge. In our trials, this happened in less than 0.5\% of all attempts. When it does happen we discard the result, and these numerical failures impact neither the results of Figure \ref{fig:gaussian}, nor the timings of the trials. 

\section{MNIST Experiment Details}\label{SM:MNIST}

\subsection{The MNIST Dataset} The MNIST Dataset \cite{lecun1998mnist} consists of a collection of $28 \times 28$ black and white images of hand-written digits. The brightness (or intensity) of each pixel is an integer between 0 and 255 with 0 being black and 255 being white. 

Each image is initially a matrix $I^m \in \mathbb{R}^{28\times 28}_+$, which is normalized to sum to one, $P^m_{ij} = I^m_{ij} / (\sum_{i',j'}I_{i',j'})$ and then converted into a point cloud 
$$\mu^m = \sum_{i,j=1}^{28} P^m_{ij}\delta_{(i,j)}.$$

\subsection{White Noise Model}

In order to apply additive noise we must generate white-noise images $\zeta$. We generate $\zeta$ as a random matrix $\mathbb{R}^{28\times 28}$ with each entry $\zeta_{ij} \sim \text{Unif}([0,1])$. This random matrix is then converted into a point cloud using the same process as is used to turn the image matrix into a point cloud.

\subsection{Linear Recovery} \label{sec:linear_rec}

The linear reconstruction in Figure \ref{fig:MNIST_IO} method is done by the following procedure.
\begin{enumerate}
    \item Given a corrupted image (which we will treat as a matrix) $\tilde{\mu}_0$ and a collection of references $\tilde{\mu}_1,...,\tilde{\mu}_p$ (also treated like matrices) estimate $\hat{\lambda}$ as
    $$\hat{\lambda} = \argmin_{\lambda \in \Delta^p} || \tilde{\mu}_0 - \sum_{i=1}^p \lambda_i \tilde{\mu}_i ||_F^2.$$
    This corresponds to projecting $\tilde{\mu}_0$ onto the convex hull of $\tilde{\mu}_1,...,\tilde{\mu}_p$ with respect to the Euclidean (equivalently Frobenius) norm.
    \item Recover $\hat{\mu}_0$ as
    $$\hat{\mu}_0 = \sum_{i=1}^p \hat{\lambda}_i \mu_i.$$
\end{enumerate}

\section{Document Classification Experiment Details} 

\subsection{Datasets} 

\begin{table*}[h!]
\centering
{\begin{tabular}{l | llllll}
\hline
\textbf{Dataset} & Description & B.O.W. Dimension & Average Words & $\mathcal{C}$ \\
\hline
\textbf{BBCSPORT} &BBC sports articles labeled by sport & 13243 & 117 & 5 \\
\textbf{20NEWS} &canonical news article dataset & 29671 & 72 & 20 \\
\hline
\end{tabular}}
\caption{Dataset characteristics.}
\label{tab:accents}
\end{table*}
We show the characteristics of the datasets in Table \ref{tab:accents}.  Bag of words (B.O.W.) dimension is the number of the unique words contained in each dataset. Average words is the average number of the unique words in a document and $\mathcal{C}$ is the number of classes of each of the datasets.

Both the BBCSport and News20 datasets that we used are those made publicly available from \cite{huang2016supervised}. Both datasets consist of two components:
\begin{enumerate}
    \item The datasets employ a word2vec \cite{mikolov2013distributed} embedding to represent the words. Each word is  represented as a point $w$ on the unit sphere in $\mathbb{R}^{300}$ and the semantic relationships between words are encoded in the geometry of the embeddings. Call the set of words $\mathcal{W}$, then the word2vec embedding is the matrix $W \in \mathbb{R}^{|\mathcal{W}| \times 300}$.
    \item Each document $D^m$ with label $y^m$ is represented by a set of tuples $\{(c^m_i,\bm w^m_i)\}_{i=1}^n$, where $n$ is the number unique words in $D^m$, $c^m_i$ is the number of times $\bm w^m_i$ shows in $D^m$.
\end{enumerate}
We use the following procedure to convert a document into an empirical measure.
\begin{enumerate}
    \item Given a document $D^m = \{(c^m_i,\bm w^m_i)\}_{i=1}^n$, let $\bm p^m \in \mathbb{R}^n$ be the probability vector with entries $p^m_i = c^m_i / (\sum_{k} c^m_k)$. 
    \item The document empirical measure is then
    $$\mu^m = \sum_{i} p^m_i \delta_{\bm w^m_i}.$$
    For convenience we filter out the points on which $p^i$ places no support.
\end{enumerate}

\subsection{Procedure for Making Figure \ref{fig:NLP}} \label{sec:NLP_figs}

As discussed in the main body, for $k \geq 2$ we select $k$ documents from each topic and 100 test documents, and then repeat 50 times and average. In both plots in Figure \ref{fig:NLP} we iterate over different choices of $k$. To reduce computational load we reuse the test sets and reference documents within each trial across choices of $k$. For example, all of the documents for $k=4$ are used for $k \geq 5$. This allows the Wasserstein distances and estimates $\hat{A}_{ij}$ to be cached avoiding further computation.

We note that this procedure produces plots in which each of the points individually has the same sampling pattern as if there was no re-use, it is only across samples that there is correlation. However, this correlation is shared across the prediction methods consider which preserves the fairness of the comparison of the methods.

\section{Possible Rank Deficiency of \textit{A}} \label{sec:rank_def}

In general, there may exist a set $\Lambda \subset \Delta^p$ such that $|\Lambda| > 1$ and for all $\lambda \in \Lambda$ we have $\mu_0 = \nu_{\lambda}$. Indeed, suppose there are index sets $I_1, I_2$ such that $I_1 \cap I_2 = \varnothing$ and $I_{1}\cup I_{2} = \{1,2,\dots, p\}$ but $\bary(\{\mu_{i}\}_{i\in I_{1}}) \cap \bary(\{\mu_{i}\}_{i\in I_{2}}) \neq \varnothing$.  If $\mu_0 \in \bary(\{\mu_{i}\}_{i\in I_{1}}) \cap \bary(\{\mu_{i}\}_{i\in I_{2}})$ then it can be expressed as $\nu_{\lambda^1} \in \bary(\{\mu_{i}\}_{i\in I_{1}})$ and $\nu_{\lambda^2} \in \bary(\{\mu_{i}\}_{i\in I_{2}}).$  Note that we may interpret $\lambda^{1}$ and $\lambda^{2}$ as vectors in $\Delta^{p}$ with supports on $I_{1}, I_{2}$, respectively.  Now let $\alpha \in [0,1]$ and let $\lambda = (1-\alpha)\lambda^1 + \alpha\lambda^2$ (with the indices modified as needed). Then we have
$$
\argmin_{\nu \in \Pac} \sum_{i=1}^p \frac{\lambda_i}{2} W_2^2(\nu, \mu_i) = \argmin_{\nu \in \Pac} (1-\alpha)\left ( \sum_{i \in I_1} \frac{\lambda^1_i}{2} W_2^2(\nu, \mu_i)\right ) + \alpha\left ( \sum_{i \in I_2} \frac{\lambda^2_i}{2} W_2^2(\nu, \mu_i)\right ) = \mu_0
$$
where the last equality follows from the fact that if $x$ minimizes both $f(x)$ and $g(x)$ then it must also minimize $\alpha f(x) + (1-\alpha) g(x)$ for $\alpha \in [0,1]$, and the assumption that $\mu_0$ is the barycenter. This demonstrates that $\mu_0$ actually corresponds to all $\lambda = \lambda = (1-\alpha)\lambda^1 + \alpha\lambda^2$ for all $\alpha \in [0,1]$. 

Below, we will give a geometric characterization for the rank deficiency of $A$ and how that relates to the optimality of the quadratic objective in Theorem \ref{thm:main}. We first recall the quadratic problem over the probability simplex: $\min_{\lambda\in\Delta^{p}} \lambda^TA\lambda$.  Since $A$ is a symmetric and positive semidefinite matrix, a Rayleigh quotient analysis yields that the minimum possible value of the objective without simplex constraints is zero, which would be attained at an eigenvector corresponding to a zero eigenvalue. Given this fact, finding an exact solution to the above quadratic problem amounts to determining whether the probability simplex intersects the eigenspace of $A$ associated with the zero eigenvalue. Hereafter, we will denote this space by $E_{0}$. 

\begin{proposition}\label{prop:NoSol}
If $E_{0}\cap \Delta^{p} =\emptyset$, there is no $\lambda \in \Delta^{p}$ such that $\lambda^TA\lambda=0$. 
\end{proposition}

\begin{proof}
   Any $\lambda\in\Delta^{p}$ realizing $\lambda^TA\lambda=0$ is an eigenvector corresponding to the zero eigenvalue.  The result follows. 
\end{proof}
Proposition \ref{prop:NoSol} shows in particular that exact minimization is only possible when $\textrm{rank}(A)<p$. Next, we state conditions under which there is a unique solution to the minimization program.
\begin{proposition}
If $\textrm{rank}(A) = p-1$ and $E_{0}\cap \Delta^{p} \neq \emptyset$, then there is a unique $\lambda \in \Delta^{p}$ such that $\lambda^TA\lambda=0$. 
\end{proposition}
\begin{proof}
    If $\textrm{rank}(A) = p-1$, $E_{0}=\{\alpha v \ | \ \alpha\in\mathbb{R}\}$ for an eigenvector $v\neq 0$ corresponding to the zero eigenvalue. Let $\lambda\in E_{0}\cap \Delta^{p}$. Then $\lambda=\alpha v$ for some $\alpha\in\mathbb{R}$.  The constraint that $\lambda\in\Delta^{p}$ determines $\alpha$ uniquely, which gives the result. 
    \end{proof}

Note that if $E_{0}$ non-trivially intersects the probability simplex and $\textrm{rank}(A)=p-1$, the condition for uniqueness is equivalent to the statement that the origin can be written as a convex combination of the columns of $A$. We next consider the remaining case when $\textrm{rank}(A)<p-1$; for simplicity of exposition, we focus on the specific case $\textrm{rank}(A)=p-2$. 
\begin{proposition}
Let $\textrm{rank} (A)=p-2$ and $v_{0}$ and $v_{1}$ be distinct eigenvectors corresponding to the zero eigenvalue of $A$. If $v_{0}\in \Delta^{p}$
and $v_{1}\in \Delta^{p}$, then there are infinitely many solutions to $\lambda^TA\lambda=0$ with $\lambda \in \Delta^{p}$. 
\end{proposition}

\begin{proof}
    By construction, $Av_{0}=0$ and $Av_{1}=0$. Let $\alpha\in [0,1]$ and define $v_{\alpha} = \alpha v_{0}+(1-\alpha)v_{1}$. Since $v_{1}$ and $v_{2}$ are in $\Delta^{p}$, it follows that $v_{\alpha}\in \Delta^{p}$. In addition, $Av_{\alpha}=\alpha Av_{1}+ (1-\alpha)Av_{2}=0$ so that $v_{\alpha}\in E_{0}$. Since $\alpha\in [0,1]$ was arbitrary, the result follows. \end{proof}
The main conclusion from the above Propositions is the following: if $\textrm{rank}(A)<p-1$, the only way to mandate a unique solution is to require that only one eigenvector corresponding to the zero eigenvalue intersects with probability simplex. By convex geometry, if two eigenvectors corresponding to the zero eigenvalue intersect the probability simplex, there will be infinitely many solutions to the quadratic program. 

\section{Example of a Karcher Mean that is not a Barycenter}
\label{sec:karcher_not_bc}
As mentioned in Section \ref{sec:Theory}, special care must be made to ensure that any Karcher mean is also a barycenter. In this section we give a simple example which illustrates this point.

Consider $S^1 = \{x \in \mathbb{R}^2 : ||x||_2 = 1 \}$ with distance given by $\rho(x,y) = \cos^{-1} \langle x, y \rangle$. Let $z_1 \neq z_2 \in S^1$ with $z_1 \neq -z_2$, and $\lambda = (1/2,1/2)$. Then the barycenter is given by $$x_{\lambda} = \frac{z_1 + z_2}{||z_1 + z_2||_2}.$$
However, the point $-\frac{z_1 + z_2}{||z_1 + z_2||_2}$ is not the barycenter but it is a Karcher mean. This is illustrated in Figure \ref{fig:karcher_not_bc}.  To show this rigorously one need only compute the Euclidean gradient and demonstrate that it is orthogonal to the tangent space.

\makeatletter
\setlength{\@fptop}{10pt}
\makeatother

\begin{figure}[t]
    \centering
    \includegraphics[width=0.3\linewidth]{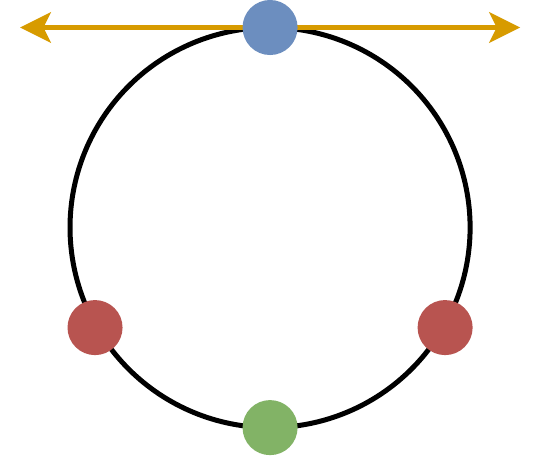}
    \caption{\emph{Red}: The points $z_{1},z_2$. \emph{Green}: The true barycenter $x_{\lambda}$. \emph{Blue}: A Karcher mean which is not the barycenter. \emph{Gold}: The tangent space at the blue point.}
    \label{fig:karcher_not_bc}
\end{figure}

\end{document}